\documentclass[submission,copyright,creativecommons]{eptcs}
\providecommand{\event}{GandALF 2023}

\usepackage{enumitem}

\usepackage{tikz}
\usetikzlibrary{arrows,automata,shapes.geometric,positioning,calc,fit,decorations.pathmorphing,shapes.misc,fadings,tikzmark}
\usetikzlibrary{decorations.pathreplacing,shapes,calligraphy,patterns,backgrounds}

\usepackage{amsmath, amsfonts, amssymb, amsthm}
\usepackage{mathtools}
\usepackage[ruled,noend]{algorithm2e}
\SetArgSty{textit} 
\SetKwComment{Comment}{/* }{ */}

\SetCommentSty{mycomfont}
\usepackage{mathpartir}
\usepackage{stmaryrd}
\usepackage{amsfonts}
\usepackage{bbold}
\usepackage{marvosym}
\usepackage{scalerel}
\usepackage{multirow,bigdelim}
\usepackage{placeins}
\usepackage{array}
\usepackage{caption}
\usepackage{subcaption}
\usepackage{pgfplots}
\usetikzlibrary{pgfplots.groupplots}
\usepackage{placeins}
\usepackage{xcolor}
\usepackage{multirow}
\usepackage{booktabs}
\usepackage{longtable}
\usepackage{cleveref}

\usepackage{thm-restate}

\usepackage{cite}
\usepackage{appendix}
\usepackage{graphicx}
\usepackage{continuous}

\definecolor{light-gray}{gray}{0.985}

\newcommand*\circled[1]{\tikz[baseline=-2pt]{
            \node[shape=circle,draw=black,thick,
    fill=light-gray,
    text=black,inner sep=2pt] (char) {#1};}}

\usepackage{float}
\usepackage{tcolorbox}

\newtheorem{definition}{Definition}
\newtheorem{remark}{Remark}
\newtheorem{example}{Example}

\newtheorem{lemma}{Lemma}

\newcommand\rquestion[1]{
\begin{center}
\begin{tcolorbox}[colback=light-gray,every float=\centering]
\begin{minipage}{\textwidth}
#1
\end{minipage}
\end{tcolorbox}
\end{center}
}

\usepackage{etoolbox}

\booltrue{result}

\pgfplotsset{compat=1.15}

\title{Conflict-Aware Active Automata Learning}
\author{Tiago Ferreira \quad L{\'e}o Henry \quad Raquel Fernandes da Silva
\institute{University College London\\ London, UK}
\email{\{t.ferreira,leo.henry,raquel.silva.20\}@ucl.ac.uk}
\and
Alexandra Silva
\institute{Cornell University\\Ithaca, NY, USA}\email{alexandra.silva@cornell.edu}
}

\def\titlerunning{Conflict-Aware Active Automata Learning}
\def\authorrunning{T. Ferreira and L. Henry and R. Fernandes da Silva and A. Silva}

\providetoggle{short}
\settoggle{short}{false}

\begin{document}

\setlength{\abovedisplayskip}{4pt}
\setlength{\belowdisplayskip}{4pt}
\setlength{\abovedisplayshortskip}{3pt}
\setlength{\belowdisplayshortskip}{3pt}

\maketitle
\begin{abstract}
Active automata learning algorithms cannot easily handle \emph{conflict} in the observation data (different outputs observed for the same inputs). This inherent inability to recover after a conflict impairs their effective applicability in scenarios where noise is present or the system under learning is mutating. 

We propose the Conflict-Aware Active Automata Learning (\CEAL) framework to enable handling conflicting information during the learning process. The core idea is to consider the so-called observation tree as a first-class citizen in the learning process. Though this idea is explored in recent work, we take it to its full effect by enabling its use with any existing learner and minimizing the number of tests performed on the system under learning, specially in the face of conflicts. We evaluate \CEAL in a large set of benchmarks, covering over 30 different realistic targets, and over 18,000 different scenarios. The results of the evaluation show that \CEAL is a suitable alternative framework for closed-box learning that can better handle noise and mutations.
\end{abstract}

\section{Introduction}
Formal methods have a long history of success in the analysis of critical systems through  abstract models. These methods are  rapidly expanding their range of applications and recent years saw an increase in industrial teams applying them to (large-scale) software~\cite{bornholt_using_2021, Cook2018, Backes2019, Chudnov2018, Le2022, Carbonneaux2022}.
The applicability of such methods is limited by the availability of good models, which require time and expert knowledge to be hand-crafted and updated. To overcome this issue, a research area on automatic inference of models, called \emph{model learning}~\cite{vaandrager_model_2017}, has gained popularity. Broadly, there are two classes of model learning: {\em passive learning}, which attempts to infer a formal model from a static log, and {\em active learning}, where interaction with the system is allowed to refine knowledge during the inference. 

In this paper, we focus on active learning, motivated by its successful use in verification tasks, e.g. in analyzing network protocol implementations, as TCP~\cite{fiterau-brostean_combining_2016}, SSH~\cite{fiterau-brostean_model_2017}, and QUIC~\cite{ferreira_prognosis_2021}, or understanding the timing behavior of modern CPUs~\cite{vila_cachequery_2020}. 
 Current state-of-the-art active learning algorithms rely on the \emph{Minimally Adequate Teacher} (MAT) framework \cite{angluin_learning_1987}, which formalizes a process with two agents: a {\em learner} and a {\em teacher}. The learner tries to infer a formal model of a system, and the teacher is omniscient of the system, being able to answer queries on potential behaviors and the correctness of the learned model. MAT assumes that the interactions between both agents are perfect and deterministic. 

\vspace{-.4cm}
\paragraph*{Learning In Practice}
Interactions with the \emph{System Under Learning} (SUL) are often non-deterministic in some way, e.g. the communications can be noisy (i.e. query answers do not only reflect the actual system output, but are instead a consequence of its interaction with the environment), or the SUL itself can change during learning. This can lead to \emph{conflicts}, which we define in the following way:
\begin{tcolorbox}[colback=light-gray]
    A \emph{conflict} appears when a query's answer formally contradicts a previous query in a way that cannot be expressed by a model of the target class.
\end{tcolorbox}
Current MAT Learners cannot handle the conflicts that arise during learning. Thus, when used in practice, MAT learner implementations use artifacts to circumvent conflicting observations. 

For example, in the case of noise, each interaction has a chance of diverging from its usual behavior. To handle this, MAT learners repeat each query $n$ times and majority-vote the result. They aim to guess an $n$ sufficiently large to prevent \emph{any} noisy observation from reaching the learner, but small enough to let the computation finish before timeout. As a consequence, noise threatens both \emph{efficiency} and \emph{correctness} of learning. We provide a framework alleviating this issue without tailoring it to specific MAT learners.

Irrespective of the nature of the conflicts detected, dealing with them requires the ability to \emph{backtrack} certain decisions that were made based on what is now considered incorrect information. This pinpoints the issue with current MAT learners: there is no notion of information storage other than the internal data structure that the learners use to build the model, which is not easily updatable in the face of conflict. This structure in fact needs to be fully rebuilt if a conflict is found, generating many superfluous (and expensive!) queries to the SUL. Separating the learning process from the information gathered through the queries allows us to \emph{retain} all the previous non-conflicting information. This alleviates the main cost of conflict handling: the unnecessary repetition of tests on the system. The Learner then only needs to rebuild its data structure based on the information already available.

\vspace{-.4cm}
\paragraph*{Contribution}
Based on the ideas above, this paper proposes the \emph{Conflict Aware Active Automata Learning} (\CEAL, pronounced \emph{seal}) framework. Any existing MAT learner 
can be used in \CEAL. When a conflict arises, we provide a method for updating the learner's internal state\,---\,without making assumptions on its data-structure\,---\,so that it remains conflict-free while removing only inconsistent information.
\begin{tcolorbox}[colback=light-gray]
	In a nutshell, this paper aims to provide classic MAT learners with a way to recover from conflicts caused by either noise or potential mutations of the system.
\end{tcolorbox}
At the heart of \CEAL is the use of an {\em observation tree}, a data structure (external to the learner) used to store information gathered from the SUL. It can be efficiently updated and used by the learner to construct its own internal data structure. When a conflict appears, we update the observation tree to reflect our knowledge, while the learner's data structure is \emph{pruned} to a conflict-free point and then expanded from the observation tree. Crucially, the learner uses \emph{the observations already stored in the tree} without requiring tests on the SUL for already observed behaviors.
\CEAL's main features are:
\begin{itemize}[leftmargin=*]
    \item The SUL is a first-class citizen, instead of being abstracted. \CEAL notably does not rely on \emph{equivalence queries}, replacing them with either a check of the stored knowledge (when sufficient) or an \emph{equivalence test}, using an \(m\)-complete testing algorithm~(e.g. the Wp-method~\cite{FBKAG91} or Hybrid-ADS~\cite{LY94,SMVJ15}).

    \item The information obtained through tests on the SUL is stored in an observation tree managed by a new \emph{Reviser} agent that is responsible for handling the conflicts and answering the learner's queries like a teacher. Providing a teacher interface is an important aspect as it enables the use of any MAT-based algorithm seamlessly, only requiring the ability to restart a classic MAT learner.
    \item The Reviser alone interacts with the SUL by means of tests meant to expand its observation tree. 
\end{itemize}
Crucially, \CEAL is less abstract than MAT, representing directly the objects and challenges of \emph{practical} active learning, while still allowing the design of Learners to enjoy the simplifying abstraction of MAT.

After some preliminaries in \Cref{sec:prelim} we formalize and prove the above claims in \Cref{sec:C3AL}. We evaluate \CEAL in \Cref{sec:exp} using a broad range of experiments\cite{neider_benchmarks_2019}. We compare several state-of-the-art algorithms (namely \lstar~\cite{angluin_learning_1987}, \kv~\cite{kearns_introduction_1994}, \ttt~\cite{bonakdarpour_ttt_2014} and \lsharp~\cite{vaandrager_new_2022}) for targets of different sizes and different levels of noise, while varying the controllable parameters for both MAT and \CEAL. The experimental results show that in the case of noise, \CEAL allows us to drastically reduce the number of repeats required to learn correct models by handling some conflicts in the information it gathers from the system. This allows \CEAL to achieve a success rate of 95.5\% compared to MAT's 79.5\% in our experiments.
\iftoggle{short}{
    
A long version of this paper, complete with the appendices presenting the proofs, some more formalization, extensive experimental results and some more discussion can be found in \cite{arxiv}. References are given when it can be useful.}{}

\section{Preliminaries}
\label{sec:prelim}
In this section, we recall Mealy machines and MAT. Fix an alphabet $A$ (a finite set of symbols). The set of finite words is denoted $A^*$, the empty word $\varepsilon$, and the set of non-empty words by $A^+$. The length of a word $w\in A^*$ is denoted $|w|$, its sets of prefixes by $\mathsf{prefixes}(w)$, its $k$-th element by $w[k]$ and the subword 
from the $i$-th to the $j$-th element by $w[i,j]$. 
The concatenation of word $w$ with symbol $a$ is denoted by $wa$. 

\vspace{-.4cm}
\paragraph*{Mealy Machines}
For the rest of the paper, we fix an input and output alphabet pair \((\Input,\Output)\). 
	A \textit{Mealy machine} over alphabets \((\Input,\Output)\) is a tuple \(\mealy=(\States,\initstate,\trans,\outf)\) where \(\States\) is a finite set of states, \(\initstate\in\States\) is the initial state, \(\trans: \States\times\Input\rightarrow \States\) is a transition function and \(\outf: \States\times\Input\rightarrow \Output\) an output function. 
    Mealy machines assign \emph{output words} ($\oword\in\Output^*$) to \emph{input words} ($\iword\in\Input^*$)\,---\,one reads input letters using $\delta$ and collects all output letters given by $\lambda$. 
    This is achieved using inductive extensions of $\trans$ and $\outf$:
	\begin{align*}
	&\trans^* \colon \States \times \Input^* \to \States  \qquad 	&&\trans^*(q,\varepsilon) = q  \qquad \trans^*(q,ia) = \trans(\trans^*(q,i),a)     \\
	&\outf^+ \colon \States \times \Input^+ \to \Output  
	\qquad&&
	 \outf^+ (q, ia) = \outf(\trans^*(q,i),a)
	\end{align*}
We now build the semantics function $\mathcal M^* \colon \Input^* \to \Output^*$ given by 
\[
\mathcal M^* (a_1 \cdots a_j) = b_1 \cdots b_j \quad \text{ where } b_k = \outf^+(q_0,a_1 \cdots a_k),\text{ for all $k=1,\ldots, j$}.
\]
Note the preservation of length of input words in output. When the functions $\delta$ and $\lambda$ are partial we call the Mealy machine $\mathcal M$ partial. A partial tree-shaped Mealy machine is called an \emph{observation tree}.

\begin{example}
\label{ex:obs-tree}
On the left below is the tree representing the tests $\{(aaa, aab), (aab, aaa), (ab, ab)\}$ and on the right the tree representing $\{(aaa, abb), (ab, ab)\}$. 
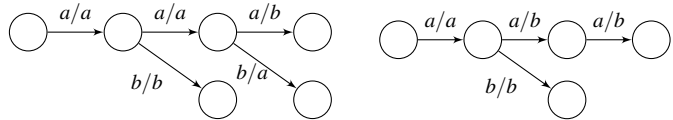
\begin{center}
{\scriptsize
    \begin{tikzpicture}[scale=.9]
        \node[ptastate] (e) at (0,0) {};
        \node[ptastate] (a) at (1.4,0) {};
        \node[ptastate] (aa) at (2.8,0) {};
            \node[ptastate] (ab) at (2.8,-1) {};
        \node[ptastate] (aaa) at (4.2,0) {};
        \node[ptastate] (aab) at (4.2,-1) {};
        \path[-{latex'}] (e)    edge node [above]  {$a / a$} (a)
                            (a)    edge node [above]        {$a/ a$} (aa)
                            (a)    edge node [below left]        {$b/ b$} (ab)
                            (aa)   edge node [above]        {$a /b$} (aaa)
                            (aa)   edge node [below left=-.2cm] {$b/ a$} (aab);
    \end{tikzpicture}	
        \qquad
    \begin{tikzpicture}[scale=.8]
        \node[ptastate] (e) at (0,0) {};
        \node[ptastate] (a) at (1.4,0) {};
        \node[ptastate] (aa) at (2.8,0) {};
            \node[ptastate] (ab) at (2.8,-1) {};
        \node[ptastate] (aaa) at (4.2,0) {};
        \path[-{latex'}] (e)    edge node [above]  {$a / a$} (a)
                            (a)    edge node [above]        {$a/ b$} (aa)
                            (a)    edge node [below left]        {$b/ b$} (ab)
                            (aa)   edge node [above]        {$a /b$} (aaa);
    \end{tikzpicture}
    }
    \vspace{-.5cm}
    \end{center}
\end{example}
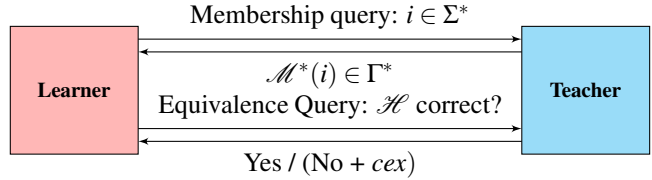
\begin{figure}[t]
\centering
\begin{tikzpicture}[scale=.85, font=\scriptsize]
		\draw[fill=learner] (0,0) rectangle (2,2);
		\draw[fill=system] (8,0) rectangle (10,2);
		\node at (1,1) {\bf Learner};
		\node at (9,1) {\bf Teacher};
		\path[-{latex'}] (2,1.8) edge node[above] {\small Membership query: $\iword \in \Input^*$} (8,1.8)
		                 (8,1.6) edge node[below] {\small $\mealy^*(\iword) \in \Output^*$} (2,1.6)
						 (2,.4) edge node[above]  {\small Equivalence Query: $\hyp$ correct?} (8,.4)
						 (8,.2) edge node[below]  {\small Yes / (No + $\mathit{cex})$} (2,.2);
\end{tikzpicture}
\caption{The Minimally Adequate Teacher framework.}\vspace*{-.4cm}
\label{fig:mat}
\end{figure}
\paragraph*{Active Model Learning} Active learning is a process in which a \emph{learner} can interact with an omniscient \emph{teacher} to build a model of an unknown system. Formally, this type of learning uses the \emph{Minimally Adequate Teacher} (MAT) framework~\cite{angluin_learning_1987} (see Fig.~\ref{fig:mat}). The teacher is supposed to have enough knowledge about the target machine \(\mealy\) to be able to answer two types of queries: 
\begin{description}
	\item[Membership] The learner sends an input word \(\iword\) to the teacher, who answers with the output word \(\mealy^*(\iword)\). 
	\item[Equivalence] The learner proposes a hypothesis model \(\hyp\). The teacher either confirms the model as correct or provides a counterexample \(\mathit{cex}\in\Input^{*}\) such that \(\hyp^*(\mathit{cex})\neq\mealy^*(\mathit{cex})\).
\end{description}

The MAT framework is an interesting abstraction to design algorithms and conduct proofs, and has been the basis for active model learning since its introduction (see e.g. \lstar~\cite{angluin_learning_1987}, \kv~\cite{kearns_introduction_1994}, \ttt~\cite{bonakdarpour_ttt_2014} or \lsharp~\cite{vaandrager_new_2022}). The teacher abstracts the system under learning (SUL), which complicates discussions on the practical interfaces between the learner and the SUL during applications. MAT does not separate the learner's core features (i.e. choosing the queries to be made and building hypotheses) from the storage of observations.
This has led the community to resort to \emph{caches}, often implemented through observation trees, to access observations directly.
Being mostly tricks to avoid repeating queries, caches are rarely discussed in the literature (although used during experiments), which had so far delayed a discussion on the practical implications of a proper handling of observations. This paper addresses this.

\paragraph*{Noise on communications} The term \emph{noise} is usually used to described a wide range (if not any form) of perturbations that can happen between the designed agent (in our case the Learner) and the SUL. 
In the case of this study we are primarily interested in the classification between \emph{input} and \emph{output} noise. 
\begin{description}
    \item[Output noise] We call output noise a perturbation that only affects what our agent sees from the world, i.e. the outputs of the SUL. Formally, this kind of noise can be represented as a non-deterministic function of \(\mealy^{*}(\iword)\) returning a different output word of same size.
    \item[Input noise] This kind of noise instead affects the query $\iword$ inputted into the SUL, so that a different input word \(\iword'\) of same size is processed instead. 
\end{description}

Noise can have different levels of \emph{structure}, being generated by different kinds of models or probability distributions. As this paper strives for a generic approach, no assumption is made on the structure of noise. Furthermore, experiments will use generic noise that has a fixed probability \emph{for each symbol} of the word, taken in sequence, to replace it with a random one according to a uniform distribution.
One notable restriction of our approach is that it does not target adversarial modifications\,---\,such as an attacker trying to change the Learner's hypotheses.
\begin{remark}
    We do not further formalize noise, as it stems for very practical considerations that may require a wide array of different formalizations. The method we propose is \emph{generic} and aims to demonstrate that paying attention to noise and \emph{conflicts} allows significant efficiency gains without any specialization towards a specific model of noise.
\end{remark}
\section{Conflict Aware Active Automata Learning}
\label{sec:C3AL}

We now introduce our alternative to MAT in practice\,---\,the \emph{Conflict-Aware Active Automata Learning} (\CEAL) framework. 
\CEAL's main features are as follows:
\begin{itemize}
    \item The SUL is a first class citizen, allowing for clearer practical discussions and modularity.
    \item The information obtained through tests on the SUL is stored in a new \emph{Reviser} agent that handles the conflicts and answers the learner's queries like a teacher.
    \item The Reviser alone interacts with the SUL by means of tests meant to expand its observation tree. The learner's queries are answered from the observation tree.
\end{itemize}

\subsection{Framework Overview}
\label{sub:overview}
\CEAL (Fig.~\ref{fig:func_view_simple}) is centered around three agents\,---\,the Learner, the System (SUL), and the Reviser\,---\,and the interfaces between them.   
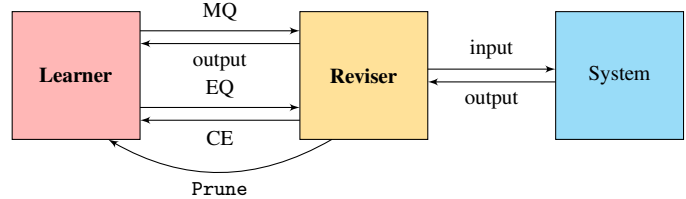
\begin{figure}[t]
    \centering
\begin{tikzpicture}[scale=.85, font=\scriptsize]
    \draw[fill=learner] (0.5,0) rectangle (2.5,2);
    \draw[fill=cache] (5,0) rectangle (7,2);
    \draw[fill=system] (9,0) rectangle (11,2);
    \node at (1.5,1) {\bf Learner};
    \node at (6,1) {\bf Reviser};
    \node at (10,1) {System};
    \path[-{latex'}] (2.5,1.7) edge node[above=0] {MQ} (5,1.7)
                     (5,1.5) edge node[below=0] { output} (2.5,1.5)
                     (2.5,.5) edge node[above=0] {EQ} (5,.5)
                     (5,0.3) edge node[below=0] {CE} (2.5,.3);
    \path[-{latex'}] (7,1.1) edge node[above=0] {input} (9,1.1)
                     (9,0.9) edge node [below=0] {output} (7,0.9);
    \path[-{latex'}] (5.5,0) edge[bend left] node [below=0] {\Restart} (2,0);
\end{tikzpicture}
\caption{Simplified view of the \CEAL framework. See Fig.~\ref{fig:MQ_implem} and~\ref{fig:EQ_implem} for more detail.}
\label{fig:func_view_simple}
\end{figure}
The {\bf Learner} plays the same general role as in MAT.  Crucially, any MAT learner can be used in \CEAL (e.g. \lstar, \kv, \ttt, \lsharp).
The Learner \emph{does not have} to store the information obtained from tests on the system. It focuses on the questions ``What is the next query to make?'' and ``How is the hypothesis built?''.
The {\bf System} is the System Under Learning, together with its environment (e.g. noise).
The {\bf Reviser} handles knowledge and conflicts. 
It answers the question ``What do we know about the system?''. It is set between the Learner and the System with interfaces to both of them. 

\CEAL is designed to improve the practical learning of reactive systems like Mealy machines. As such, it makes use of features that are core to such models, like causality and closure under inputs and outputs. However, the main ideas behind \CEAL's philosophy and separation of concerns can be adapted to learn other types of automata, such as acceptors like DFAs.

\begin{remark}
    The Reviser acts as a MAT teacher w.r.t. the Learner, answering membership and equivalence queries, with the added ability to \Restart the learner to place it in a state coherent with the Reviser's information.
    On the System's view, the Reviser acts as a tester, providing input sequences (tests) and recording the system output. The outside views of the learner and system in \CEAL are illustrated below.
    \vspace{-.2cm}
    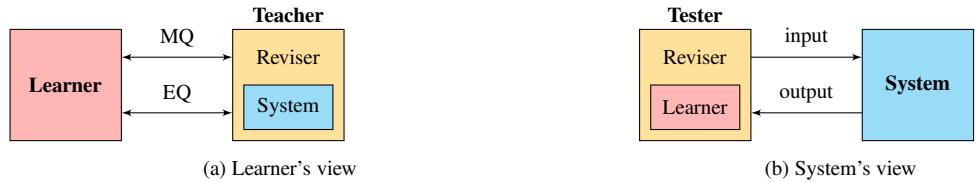
\begin{figure}[H]
    \centering
    \hfill
    \hspace{1cm}
    \begin{subfigure}{.45\textwidth}
    \begin{tikzpicture}[scale=.74,font=\scriptsize]
    
        \draw[fill=learner] (0,0) rectangle (2,2);
        \draw[fill=cache] (4,0) rectangle (6,2);
        \draw[fill=system] (4.2,0.2) rectangle (5.8,1);
        \node at (1,1) {\bf Learner};
        \node at (5,.6) {System};
        \node at (5,1.5) {Reviser};
        \node at (5,2.25) {\bf Teacher};
        \path[{latex'}-{latex'}] (2,1.5) edge node[above] {MQ} (4,1.5)
        (2,.5) edge node[above] {EQ} (4,.5);
    \end{tikzpicture}
    \caption{Learner's view}
    \label{fig:leaner_view}
    \end{subfigure}
    \hfill
    \begin{subfigure}{.45\textwidth}
    \hspace{-1cm}
    \centering
    \begin{tikzpicture}[scale=.74,font=\scriptsize]
        \draw[fill=cache] (0,0) rectangle (2,2);
        \draw[fill=learner] (.2,0.2) rectangle (1.8,1);
        \draw[fill=system] (4,0) rectangle (6,2);
        \node at (1,.6) {Learner};
        \node at (5,1) {\bf System};
        \node at (1,1.5) {Reviser};
        \node at (1,2.25) {\bf Tester};
        \path[-{latex'}] (2,1.5) edge node[above] {input} (4,1.5)
        (4,.5) edge node [above] {output} (2,.5);
    \end{tikzpicture}
    \caption{System's view}
    \label{fig:system_view}
    \end{subfigure}
    \vspace{.2cm}
    \label{fig:views}
    \caption{Reviser's interfaces}
    \end{figure}    
\end{remark}

The {\bf interfaces} on the Learner side are similar to MAT: the Learner can perform membership queries (MQ) and equivalence queries (EQ) on the Reviser, with the latter potentially resulting in a counterexample (CE). Note that the queries are sent to the Reviser and not directly to the SUL: a crucial design choice. This allows us to control the information that the Learner obtains, and reuse the information in the Reviser with no new tests.
Formally, \CEAL provides the following functions as module interfaces: 
\begin{itemize}
    \item \(\MQ: \Input^*\to (\Output^*\cup\{\Restart\})\) the membership query of the learner that the Reviser has to implement. It varies from the MAT function as the Reviser may return a command to prune the learner's state instead of an output.
    \item \(\EQ: \Mealy \to ((\Input^*\times\Output^*)\cup\{\Restart\})\) the equivalence query of the learner that the Reviser has to implement. It may return ``\Restart" instead of ``Yes".
    \item \(\System: \Input^*\to (\Input^*\times\Output^*)\) is a call to the system for a specific test. The system returns the corresponding behavior (input and output), with the effect of noise applied.
\end{itemize}
In the interface mentioned above, an $\EQ$ can never return ``Yes" as in MAT. 
This work is left to the Reviser, that will halt the learning process according to the termination criterion chosen (see Section~\ref{sub:reviser}).
The Prune signal does not require us to modify the code of a MAT Learner, as it can be implemented by restarting the Learner without requiring further access to the Learner's internals. The Reviser's caching of observations ensures that this operation does not add to the query complexity of the process.
\begin{remark}
    The main cost of learning comes from \emph{unit interactions with the system}\,---\,each individual symbol that is inputted into or outputted by the system\,---\,as these tests are generally costly to perform and that cost cannot be compensated.
\end{remark}
\subsection{The Reviser}\label{sub:reviser}
The Reviser agent is the core of the \CEAL framework. It concretizes its main idea: taking the storage and handling of observations \emph{out} of the Learner's prerogatives.
Its task is to update the observation tree $\BS$ on which a Learner is trained. We will assume the following interface is made available to the Reviser:

\begin{definition}[Operations on observation trees]\label{def:operations}
Given an observation tree $\BS$, we define the following functions to access and modify $\BS$:
\begin{itemize}
	\item $\Lookup_{\BS} : \Input^* \to (\Output^* \cup \{ \Null \})$ receives an input word $i$ and returns output $o$ if $(i,o)$ is present in the observation tree $\BS$. Otherwise, $\Null$ is returned.
	\item $\Update_{\BS} : (\Input^* \times \Output^*) \to \mathbb{2}$ updates the observation tree $\BS$ to take into account a new query pair, revoking conflicting information if necessary. Returns $\top$ if the new information conflicts with \BS.
\end{itemize}
\end{definition}

Note that the function $\Update$ is the only one that alters the tree and handles conflicts. Implementations of these functions are given in \iftoggle{short}{\cite[Appendix A]{arxiv}}{\Cref{app:trees}}.
Using these functions, we can now define the \emph{language} of an observation tree as the set of observations that it can transmit to a Learner, and provide a formal definition of a conflict as a non-additive change to the language of \BS.
\begin{definition}
Given an observation tree \BS, we call \emph{language} of \BS the following set 
\[\lang_\BS=\{(\iword,\Lookup_\BS(\iword))\in\Input^* \times \Output^* \mid \Lookup_\BS(\iword)\in\Output^* \}\ .\]
\end{definition}
\begin{definition}
    An observation \((\iword,\oword)\) \emph{conflicts} with an observation tree \(\BS\) when the tree \(\BSb\) obtained by calling \(\Update_{\BS}(\iword,\oword)\) satisfies
    \(\exists(\iword',\oword')\in\lang_\BS,\ \oword'\neq \Lookup_{\BSb}(\iword')\).
    Two observations \((\iword,\oword)\) and \((\iword',\oword')\) \emph{conflict}, written 
    \((\iword,\oword)\conflicts(\iword',\oword')\), when there is an input word 
	\(\iword''\in\pref(\iword)\cap\pref(\iword')\) such that \(\oword[|\iword''|]\neq\oword'[|\iword''|]\).
    \end{definition}
Note that a conflict appears not between the System and the Reviser, but signifies that the Reviser wants to update its answer to some information previously given to the Learner.

\begin{definition}[Reviser]\label{def:reviser}
The Reviser contains an observation tree $\BS$ and implements four operations:
\begin{itemize}
 \item[\circled 1] $\Apply_{\BS} : (\Input^* \times \Output^*) \to (\Output^*\cup\{\Restart\})$ updates \BS with the observation gained from a system test. It then either returns the query output or prunes the learner if a conflict is detected.
 \item[\circled 2] 
$\Read_{\BS} : \Input^* \to (\Output^*\cup\{\Restart\})$ looks in $\BS$ for a query answer and either returns it or tests the system if necessary. Note that if a test is performed, then \(\BS\) is updated accordingly. 
    
   \begin{minipage}[t]{0.45\textwidth}
    \vspace*{-0pt}\hspace*{-.5cm}
    \centering
	\begin{algorithm}[H]
        \caption{\(\Apply_{\BS}(i,o)\)}
        \label{alg:apply}
        \KwData{$(i,o)$ trace from the SUL.}
        \If{$\Update_{\BS}(i,o)$}{
            \Return \Restart\;
        }
        \Return{$o$}\;
    \end{algorithm}
\end{minipage}
\hfill
\begin{minipage}[t]{0.475\textwidth}
    \vspace{-0pt}
    \centering
    \begin{algorithm}[H]
        \caption{\(\Read_{\BS}(i)\)}
        \label{alg:read}
        \KwData{The queried string $i$.}
        $o \gets \Lookup_{\BS}(i)$\;
        \lIf{$o \not = \Null$}{\Return{$o$}
        }
        \Return{$\Apply_{\BS}(
        \System(i))$}\;
    \end{algorithm}	
\end{minipage}
\item[\circled 3] $\BCheck_{\BS} : \Mealy \to ( (\Input^* \times \Output^*) \cup \{ \Null \})$ performs a consistency check of a given Mealy machine hypothesis against the observation tree $\BS$. 
   Returns a counterexample if found or $\Null$ if no divergences are found.

\item[\circled 4] $\Test_{\BS} : \Mealy \to ( (\Input^* \times \Output^*) \cup \{ \Restart \})$ is the function used to look for counterexamples in the System. It takes a hypothesis proposed by the learner and coherent with the observation tree, and tests the SUL until a counterexample or a conflict is found. The tests are taken from \(\sampleWord\) which is instantiated by an off-the-shelf test suite generating algorithm (e.g. the Wp-method~\cite{FBKAG91} or Hybrid-ADS~\cite{LY94,SMVJ15}) in practice.

\begin{minipage}[t]{0.37\textwidth}
    \vspace{0pt}\hspace*{-.5cm}
    \centering
    \begin{algorithm}[H]
		\caption{$\BCheck_{\BS}(\hyp)$}
		\label{alg:check}
		\KwData{Hypothesis $\hyp$}
		\For{$(i, o) \in \BS$}{
			\If{$\hyp^*(i) \not = o$}{
				\Return{$(i, o)$}\;
			}
		}
		\Return{$\Null$}\;
	\end{algorithm}
\end{minipage}
\hfill
\begin{minipage}[t]{0.6\textwidth}
    \vspace{0pt}
    \centering
	\begin{algorithm}[H]
        \caption{\(\Test_{\BS}(\hyp)\)}
        \label{alg:rand_test}
        \KwData{A hypothesis \hyp coherent with $\BS$.
        }
        \While{$\top$}{
        $w \gets \sampleWord()$\;
        \((i,o)\gets\System(w)\)\;
            \lIf{$\Apply_{\BS}(i,o) = \Restart$}{  \Return{\Restart}
            }
            \lIf{$\hyp(i)\neq o$}{
                \Return{(i,o)}
            }
        }
    \end{algorithm}
    \end{minipage}  
\end{itemize}
\end{definition}
Crucially, the above functions rely on the observation tree's interface to handle the conflict as they arise, forwarding the \Restart command to the Learner when needed. 

\paragraph*{Update Strategies}
\label{par:updates}
At the core of dealing with conflicts is the idea of identifying information that will be sacrificed for the sake of cohesion. The way this is achieved depends largely on the type of conflict, and the \emph{meaning} of observing such a conflict. We propose two ways to resolving conflicts in \CEAL:

\smallskip
\noindent \circled{\footnotesize 1} \texttt{Most Recent}: When a conflict is identified, the most recently observed (freshest) query information is committed to the observation tree, and the previous one suppressed, if needed. This approach makes sense, for example if the target system has mutated and we are only interested in capturing the most up-to-date behavior, or as a base default strategy. We define \(\pref(\iword,\oword)=\{(\iword[1,n],\oword[1,n])\mid 0\leq n\leq |\iword|\}\).
	
\begin{restatable}{proposition}{mostrecent}
    \label{prop:inv_rec}
In the case of the \texttt{Most Recent} update strategy, given a stream of tests \(((\iword_k,\oword_k)_{k\in\bbN})\), at any step \(K\in\bbN\):   
\(
    \lang_T=
    \{\pref(\iword_k,\oword_k)\mid 0\leq k\leq K\ \wedge\ \not\exists k<l\leq K,\text{ s.t. } (\iword_k,\oword_k)\conflicts(\iword_l,\oword_l)\}
\).
\end{restatable}

\begin{example}
In \Cref{ex:obs-tree}, the right-hand tree is the result of observing \((aaa, abb)\) starting from the left-hand tree. Notice that the sets prefixes of the sets of observations in \Cref{ex:obs-tree} verify \Cref{prop:inv_rec}.
\end{example}

\noindent \circled{\footnotesize 2} \texttt{Most Frequent}: When two possible output sequences conflict for a given input sequence,
the most frequently observed one is returned to the Learner. This information can be obtained passively by keeping track of naturally occurring repetitions of queries, or actively by specifying a sample size on which the frequency is estimated. This approach makes sense for example for conflicts that are due to unwanted statistical noise in the observations.

We define 
\(\Count(\iword,\oword)=|\{k\mid (\iword,\oword)\in\calP(\pref)(\{(\iword_k,\oword_k)_{k\in K}\})\}|\) as the number of observations of which \((\iword,\oword)\) is a prefix in an observation stream \((\iword_k,\oword_k)_{k\in\bbN}\) considered at step $K$.

\begin{restatable}{proposition}{mostfrequent}
    \label{prop:inv_freq}
    In the case of the \texttt{Most Frequent} update strategy, given a stream of tests \(((\iword_k,\oword_k)_{k\in\bbN})\), at any step \(K\in\bbN\):   
\begin{align*}
\mathit{mf}((\iword_k,\oword_k), (\iword_l,\oword_l)) &\triangleq \Count(\iword_k,\oword_k)<\Count(\iword_l,\oword_l)\ \vee \ (\Count(\iword_k,\oword_k)=\Count(\iword_l,\oword_l)\ \wedge\ k<l)\\
\lang_\BS
&=\{\pref(\iword_k,\oword_k)\mid k\leq K\ \wedge\ \not\exists k<l\leq K,\text{ s.t. } (\iword_k,\oword_k)\conflicts(\iword_l,\oword_l)\ \wedge \mathit{mf}((\iword_k,\oword_k), (\iword_l,\oword_l))\}\   
\end{align*}
\end{restatable}

We present implementations of \Update and \Lookup fitting these two strategies in \iftoggle{short}{\cite[Appendix A]{arxiv}}{\Cref{app:trees}}, and proofs of the above properties in \iftoggle{short}{\cite[Appendix B]{arxiv}}{\Cref{app:proof}}.

\begin{remark}
    An observation \((\iword,\oword)\) conflicting with an observation tree \(\BS\) implies that \((\iword,\oword)\conflicts(\iword',\oword')\) for some \((\iword',\oword')\in\lang_\BS\). The other implication is not always true, e.g. for the \texttt{Most Frequent} update strategy.
\end{remark}

\paragraph*{Termination} The termination criteria of \CEAL are the same as those used in MAT in practice: in our experiment, we terminate when our currently selected hypothesis has survived for a fixed number of tests that is deemed sufficient, or if a predefined limit number of queries is reached.

\paragraph*{Hypothesis Selection}
Active automata learning involves the production of a sequence of hypotheses that are refined over time, with the goal of converging towards a correct one. As such, a key characteristic of different approaches to automata learning is how a final model is to be selected, out of the many hypotheses. In the case of MAT this is simple: Learning produces a sequence of ever more accurate models, until termination occurs with a positive equivalence query. It is then logical to pick the most recently produced hypothesis as the final model. However, when dealing with conflicts and different update strategies, this is no longer necessarily the case for \CEAL. In particular, when it comes to electing a model out of a sequence of hypothesis, \CEAL has two options:

\begin{itemize}[leftmargin=*]
	\item \texttt{Most Recent}: This hypothesis selection strategy is the one known classically: the most recently produced hypothesis is the one to be elected as final. This strategy is sensible in the case of learning with no noise, or in the case of learning targets that evolve over time.
	\item \texttt{Most Frequent}: In this selection strategy, the sequence of hypotheses is analyzed to elect a final model. We count the frequency of each unique model (up to language equivalence) over the sequence, and elect the most frequently occurring one. This strategy makes sense when dealing with noise, as we may be producing (rarely) hypotheses that capture noisy behavior that is fixed over time. As such, we want to select not the latest model produced, but the one that is the most stable. This strategy can be implemented efficiently in practice (using hash fingerprints and counters, for example) and on-the-fly during learning, allowing us to not have to store the whole sequence of hypotheses as it is produced.
\end{itemize}

\subsection{Interface Implementation}\label{sub:impl} 

We now explain how to build the interface described in Sec.~\ref{sub:overview} using the Reviser. 
This mostly amounts to implementing membership and equivalence queries, as the testing interface is simply composed of calls to \(\System\). 
{\bf Membership} queries can be defined, for $i\in \Input^*$, as \(\MQ(i)=\Read_{\BS}(i)\). 
When $\BS$ does not have the answer to this particular query, \(\Read_{\BS}\) sends it through to the SUL (with the call to \(\System\)) and the result is applied in $\BS$. This process is illustrated in Fig.~\ref{fig:MQ_implem}.
\begin{figure}[H]
\centering\vspace*{-.4cm}
 \begin{tikzpicture}[scale=.85]
    \draw[fill=learner] (0,0) rectangle (2,2);
    \draw[fill=cache] (4,0) rectangle (6.1,2);
    \draw (4.2,1.2) rectangle (5.9,1.8);
    \draw (4.2,0.2) rectangle (5.9,.8);
    \draw[fill=system] (8,0) rectangle (10,2);
    \node at (1,1) {\bf Learner};
    \node at (5,2.25) {\bf Reviser};
    \node at (5.05,1.5) {$\Read_{\BS}$};
    \node at (5.05,.5) {$\Apply_{\!\BS}$};
    \node at (9,1) {\bf System};
    \node at (3,1) {o};
    \path[-{latex'}] (2,1.6) edge node[above] {$i$} (4.2,1.6)
                     (4.2,1.4) edge[out=180,in=0] (2,1)
                     (4.2,0.6) edge[out=180,in=0] (2,1)
                     (4.2,0.4) edge node[below] {$\Restart$} (2,.4);
    \path[-{latex'}] (5.9,1.5) edge[out=0,in=180] node[above] {$i$} (8,1.5)
                     (8,0.5) edge[out=180,in=0] node[below] {$(i,o)$} (5.9,0.5);
\end{tikzpicture} 
\caption{Implementation of Membership Queries (\MQ) in the Reviser.}
\label{fig:MQ_implem}  
\end{figure}
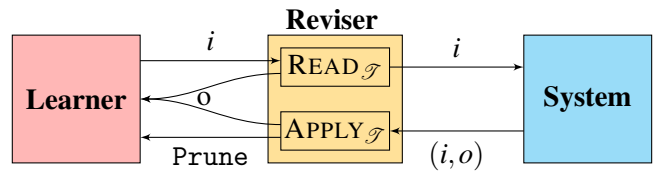

\vspace*{-.3cm}
\noindent{\bf Equivalence} queries are handled in two steps. First, the hypothesis given by the learner is checked against the observation tree using \(\BCheck_{\BS}\). If a counterexample is found, it is returned. Otherwise, the Reviser tests the System to discover new information and update the tree. 
If a counterexample is found, either it is returned to the learner or, if a conflict arose, the learner is pruned. (Fig.~\ref{fig:EQ_implem}).
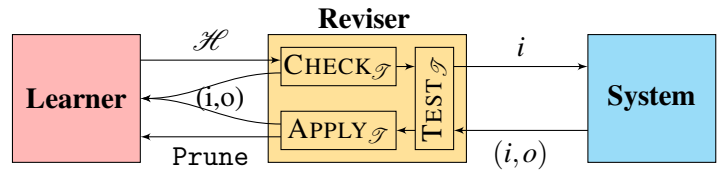
\begin{figure}[H]
\centering\vspace*{-.4cm}
    \begin{tikzpicture}[scale=.85]
        \draw[fill=learner] (0,0) rectangle (2,2);
        \draw[fill=cache] (4,0) rectangle (7.1,2);
        \draw (4.2,1.2) rectangle (6,1.8);
        \draw (4.2,0.2) rectangle (6,.8);
        \draw (6.3,0.2) rectangle (6.9,1.8);
        \draw[fill=system] (9,0) rectangle (11,2);
        \node at (1,1) {\bf Learner};
        \node at (5.5,2.25) {\bf Reviser};
        \node at (5.1,1.5) {$\BCheck_{\!\BS}$};
        \node at (5.1,.5) {$\Apply_{\!\BS}$};
        \node at (6.6,1) {\rotatebox{90}{$\Test_{\BS}$}};
        \node at (10,1) {\bf System};
        \node at (3.25,1) {(i,o)};
        \path[-{latex'}] (6,1.5) edge (6.3,1.5)
                         (6.3,0.5) edge (6,0.5); 
        \path[-{latex'}] (2,1.6) edge node[above] {\hyp} (4.2,1.6)
                         (4.2,1.4) edge[out=180,in=0] (2,1)
                         (4.2,0.6) edge[out=180,in=0] (2,1)
                         (4.2,0.4) edge node[below] {$\Restart$} (2,.4);
        \path[-{latex'}] (6.9,1.5) edge[out=0,in=180] node[above] {$i$} (9,1.5)
                         (9,0.5) edge[out=180,in=0] node[below] {$(i,o)$} (6.9,0.5);
    \end{tikzpicture}  
\caption{Implementation of Equivalence Queries (\EQ) in the Reviser.}
\label{fig:EQ_implem}  
\end{figure}

\begin{remark}[Modularity]
We present \CEAL in the case of \emph{black-box} learning where we do not have access to any information from the system but can interact with it. However, the framework is fully modular, as can be seen from the high-level functions presented in Section~\ref{sub:reviser}: one can interface any model-checking (or other) method before the calls to system in \(\Test_{\BS}\) and \(\Read_{\BS}\) when related models (specifications, parts of the System\dots) are available, allowing \CEAL to perform gray-box or even white-box learning. \CEAL focuses on the \emph{storage} of information, without restrictions on its acquisition.
\end{remark}

\paragraph*{Correctness}
\Cref{prop:inv_rec} and \Cref{prop:inv_freq} characterize the language of the Reviser, and the following results describe its interactions with the Learner and the System (as proved in \iftoggle{short}{\cite[Appendix B]{arxiv}}{\Cref{app:proof}}).

\begin{restatable}{lemma}{BSandLang}
    \label{lem:BSandLang}
    During an execution of \CEAL, all tests on the System are integrated in \(\BS\) through \(\Update_\BS\), and the Learner queries are answered according to \(\lang_\BS\).
\end{restatable}
\vspace{-.1cm}
\begin{restatable}{proposition}{prune}
    \Restart is sent to the Learner exactly when a new observation conflicts with \BS.
\end{restatable}

\subsection{Optimizations}
\label{sub:opti}
\CEAL allows us to implement two main optimizations, both in the framework and its algorithms.
\begin{description}
    \item[Query Caching]
    One of the direct benefits of the presence of the Reviser is that the learner does not have to cache membership queries to avoid repeating them, as it obtains knowledge through the Reviser's data structure. 
    It especially offers a good basis to discuss algorithms that are based on the observation tree themselves~\cite{SB19,vaandrager_new_2022}, for which the Learner's data structure is very limited.
    \item[Specialized Pruning]
    In order to fully support the simplistic interface of a classic MAT learner, we restart it\,---\,at no extra query cost (see~\Cref{sub:overview})\,---\,when the \(\Restart\) signal is sent. For specific Learner data structures, the time-complexity of this operation can be enhanced by suppressing only the part of the data structure that is impacted by the conflict (instead of restarting it completely). This optimization, however, will be specific to each learning algorithm. In the same way as the Learner can read the Reviser' observation tree, a \CEAL-specific Learner could compute its internal data-structure directly from it after a pruning without requiring restarts.
\end{description}

\section{Evaluation}
\label{sec:exp}
We introduced \CEAL to extend the power of classic MAT learners into environments that may cause conflicts, for instance caused by \emph{noise}. In practice, each symbol inputted in or outputted by the SUL can be noisy, making longer queries more likely to have noisy results overall. Recall that poorly guessing the number $n$ of query repetitions can lead to a learning failure (if noise is integrated in the system), or a timeout (if too many repetitions are made). 
Hence noise does not only affect the \emph{efficiency}, but also the \emph{success rate} of the learning, i.e., the proportion of runs that end with a correct hypothesis.

It is then important to measure how well \CEAL can avert these negative effects. In particular, we are interested in testing if \CEAL's approach is sufficient to allow for an improvement of learning environments plagued by noise\footnote{We compare success rates and number of tests issued instead of  running times as to make hardware-agnostic benchmarks that capture the main factors in both efficiency and correctness.}. We evaluate this through the following research question:
\rquestion{ Across different realistic model learning targets, types of noise, and noise levels, does \CEAL provide a better learning environment in terms of both success rates and number of tests issued when compared to the state-of-the-art MAT-based approaches?}

\subsection{Experiments}
We first present the experimental setup and the outlines of the experimental results. We focus on the difference between \emph{uncontrollable parameters}, i.e., that are part of the target and its environment and can't be altered by the learning setup, and the \emph{controllable} ones, that are chosen when designing a learning session.
\emph{Uncontrollable parameters} are related to the SUL and its environment.
\begin{description}
	\item[Realistic Targets] We run our experiments over a range of 36 Mealy machines representing real world systems from previous successful model learning applications~\cite{neider_benchmarks_2019}. These range in size between 4 and 66 states, with alphabet sizes between 7 and 22 input symbols.
	\item[Different Types of Noise] We run the targets on different types of realistic simulated noise, namely input and output noise as described above.
	\item[Different Levels of Noise] We run the above mentioned noise types over 3 different levels: 0.01, 0.05, and 0.1. These indicate the probability that each symbol has of being noisy.
\end{description}
Given the above constraints, we aim to reach the best performing learning session by manipulating the following \emph{controllable parameters}:

\vspace{-.1cm}
\begin{description}
	\item[Framework] We run each experiment under both MAT and \CEAL.
	\item[Algorithm] We run each experiment under \lsharp, \ttt, \kv, and \lstar. We use the implementations of these algorithms provided in LearnLib~\cite{isberner_open-source_2015}. Notably, \lstar is implemented with Rivest \& Schapire's improvements~\cite{rivest_inference_1993} and we re-implemented \lsharp completely to incorporate it into the LearnLib library.
	\item[Number of Repeats] We compare different numbers of repeats used in majority voting test results to remove noise in MAT, or in sampling frequencies for the update strategy in \CEAL. We use one of 3 different levels of repeats, in pairs of $(\mathit{min\_repeats}, \mathit{max\_repeats})$: $(5, 10), (10, 20), (20, 30)$\footnote{Each test is repeated $\mathit{min\_repeats}$ times and then if at least 80\% of the queries agree the result is returned. Else, the query is repeated until it is the case or $\mathit{max\_repeats}$ is reached at which point the majority answer is returned.}.
\end{description}

\vspace{-.1cm}

Each experiment uses the following settings to enhance its learning, independent of the above mentioned variables. Firstly, caching of previously observed queries is done wherever possible in MAT, and the \texttt{Most Recent} update and \texttt{Most Frequent} hypothesis selection strategies are used in \CEAL, for simplicity. Secondly, the Hybrid-ADS~\cite{LY94} equivalence testing algorithm is used for all runs as we found it to be the best performing for our experiments. Thirdly, each independent experiment is performed with 100 runs, and its results averaged for consistency. And finally, each run is allowed to use up to 10 million queries before an unsuccessful timeout is declared.

Due to the vast number of variables considered, we are unable to fully describe the result of the over 18,000 distinct experiments, and close to 2 million runs that we have performed. However this is not required to rigorously answer our research question. What we have to consider is, for each combination of our independent variables (target, noise type, and noise level), which framework allows for the most efficient learning configuration. 

The graphs below (Figs. \ref{fig:result-INPUT-0.01} - \ref{fig:result-OUTPUT-0.1}) summarize, for each target and for all levels and types of noise, the success rates and number of symbols tested of the best controllable parameter profile for both MAT and \CEAL. We have also included all the data used, as well as conclusions in \iftoggle{short}{\cite[Appendix C]{arxiv}}{\Cref{app:results}}.

\resultsGraph{INPUT}{0.01}{p}
\resultsGraph{INPUT}{0.05}{p}
\resultsGraph{INPUT}{0.1}{p}
\resultsGraph{OUTPUT}{0.01}{p}
\resultsGraph{OUTPUT}{0.05}{p}
\resultsGraph{OUTPUT}{0.1}{p}

\subsection{Analysis}
We now analyze the results of these experiments to draw some high-level conclusions about how MAT and \CEAL compare and answer our research question.
\paragraph*{Success Rates}
First and foremost, we discuss the impact of different parameters on the success rates of the experiments.
We can see from the graphs (Figs. \ref{fig:result-INPUT-0.01-success} - \ref{fig:result-OUTPUT-0.1-success}) that, as expected, while the exact type of noise does not have a significant impact on success rates and test counts, the \emph{level of noise} does. In particular, at a very low level of 0.01\% (Figs. \ref{fig:result-INPUT-0.01-success}, \ref{fig:result-OUTPUT-0.01-success}) both frameworks are capable of maintaining perfect success rates. However, once the level increases to 0.05\% (Figs. \ref{fig:result-INPUT-0.05-success}, \ref{fig:result-OUTPUT-0.05-success}), MAT's success rate starts to fluctuate, more so in bigger targets. \CEAL too seems to be slightly affected by an increase in noise, but overall maintains a success rate close to 100\%. Once the noise is increased to its highest level, 0.1\% (Figs. \ref{fig:result-INPUT-0.1-success}, \ref{fig:result-OUTPUT-0.1-success}), we can see that MAT's success rates reduce significantly, while \CEAL's tend to stay high for a great number of targets, until they inevitably decrease when faced with massive targets at this level of noise. 
\textbf{\CEAL manages to stay consistently reliable in the face of these large alphabets}.

\paragraph*{Efficiency}
Let us now turn our attention to the system test count graphs (Figs. \ref{fig:result-INPUT-0.01-symbols} - \ref{fig:result-OUTPUT-0.1-symbols}). Overall we see an expected picture: Larger systems require more tests to be learned. A particular caveat to notice however, is that while MAT appears to have quite efficient runs on large noisy targets, their respective success rates are considerably lower. Although efficiency of learning is certainly important, it is of low use if at the end the reported hypothesis is not correct. This result is expected: If a learning run fails due to, for example, high noise not being fully filtered out, the MAT learner will collapse before it finishes running. This leaves a final test count that is quite low, but also gives us an incorrect hypothesis.

\paragraph*{Overall Results}
Perhaps most importantly, \textbf{\CEAL provides the most efficient \emph{correct} configuration in 70\% of the experiments}, having a better success rate than MAT or the same with a lower average number of tests used. We provide this result for each individual experiment in \iftoggle{short}{\cite[Appendix C]{arxiv}}{\Cref{app:results}}. In particular, in every experiment \CEAL performs with a success rate that is at least as high as MAT's, often outperforming it. In addition to this, experiments ran with \CEAL had an overall success rate of 95.5\% compared to MAT's 79.5\% success rate. This alone has allowed \CEAL to perform successful runs that no configuration of MAT was able to perform, namely learning moderate to large targets at 0.1\% noise.

We found that a lot of the improvements provided by \CEAL are commonly a consequence of it being more successful when running at a \emph{lower number of repeats} when compared to the ones required by MAT. This solidifies our initial hypothesis of there being a benefit in reducing the number of repeats used when learning noisy targets. The above provides enough supporting evidence to answer our research question positively: \textbf{\CEAL provides a better learning environment in terms of both success rates and number of tests issued when compared to the state-of-the-art MAT-based approaches}.

\paragraph*{Other Findings}
We report on other interesting findings in \iftoggle{short}{\cite[Appendix D]{arxiv}}{\Cref{app:analysis}}. Two results, however, are of particular significance:
As already accepted by the community~\cite{aslam2020}, we can confirm that indeed most of the tests spent in learning are used to realize equivalence queries. In particular, we found that \textbf{equivalence tests account for 89.1\% of tests in MAT, 59.8\% in \CEAL, and 66.3\% on average}.

Perhaps more surprisingly, the commonly accepted ordering of Learner efficiencies did not surface in our experiments. From theoretically most performant to least we have \lsharp, \ttt, \kv, and \lstar, based on complexity analyses in MAT. Through our experiments we have found that, at least in our particular case of black-box learning (i.e. learning from SUL tests only) of Mealy machines with noise and randomized testing algorithms, this ordering cannot be seen in the results. The best configurations for each framework were not consistent on which algorithm performed the best. Not only was there no clear "winning" algorithm, we found no pattern based on noise, target and alphabet size, or number of repeats that had a strong enough correlation to the better performance of any one algorithm.

We believe that this is not a flaw of the complexity analyses themselves. It is simply that complexity analyses in MAT abstract away the biggest cost of learning: equivalence tests. It may be that more recent algorithms have a theoretical (and membership query based) advantage over classic algorithms, however the nature of randomized equivalence oracles seems to be a bigger agent of chaos, and a good or bad run of the equivalence oracle quickly overshadows the small advantage that some algorithms may have.

\section{Related Work}
\label{sec:sofa}
There has been extensive work on finding ways of applying classic learning algorithms like \lstar~\cite{angluin_learning_1987}, \kv~\cite{kearns_introduction_1994}, \ttt~\cite{bonakdarpour_ttt_2014}, and \lsharp~\cite{vaandrager_new_2022} to real world systems such as  passports~\cite{aarts_inference_2010}, network protocol implementations~\cite{fiterau-brostean_combining_2016,fiterau-brostean_model_2017,ferreira_prognosis_2021}, and bank cards~\cite{aarts_formal_2013}. All these works rely on ad-hoc implementations of noise handling which is inefficient and not formalized in the MAT framework. One of the goals of our framework, which replaces the teacher with the SUL and the Reviser, is to discuss how noise can be dealt with in the learning process, independently of the type of Learner being used. The LearnLib library~\cite{isberner_open-source_2015} provides \emph{caches} that can be placed in the learning environment to avoid the repetition of queries, much like observation trees. Note that our Reviser agent goes further than LearnLib as it provides the ability to act as membership and equivalence oracles, test the system, and act on conflicts by pruning the learner's data structure in an efficient (query-wise) and correct manner.

There has also been previous work in improving the efficiency of model learning strategies for mutating targets by reusing previously learned behavior, using \emph{adaptive} learning algorithms~\cite{damasceno_learning_2019, howar_adaptive_2018,windmuller_active_2013,chaki_verification_2008,Ferreira2022}. These algorithms work by being able to start learning with pre-seeded information of previous runs that has been confirmed to still apply in the current target, or by being able to filter this information themselves if it is found to no longer apply. Additionally, there has been some work on \emph{Lifelong Learning}~\cite{ahrendt_lifelong_2022}, where model learning and model checking are used together to run over the development lifecycle of a system. This allows for the quick discovery of bugs in the development cycle. However, when these are found and corrected, learning needs to be \emph{manually restarted}.

Our model learning framework improves on these two lines of work, being able to autonomously correct itself when faced with conflicts. It can do so without any notification of mutations in the system, allowing it to be applied to complete closed-box systems, unlike the current state-of-the-art adaptive algorithm~\cite{Ferreira2022}. Additionally, it is capable of continuously checking for changes in the system, much like Lifelong Learning, but requiring no human interaction on system changes. These characteristics make it resilient to real world noise, allowing the learner to correct itself as it identifies the correct behavior.

Our work participates in the current trend trying to link learning to testing, which spans communities, e.g. formal approaches~\cite{daghstul_learn_test}, genetic approaches~\cite{TALL19}, and fuzzing~\cite{Zeller2021}. In this context, \CEAL provides a modular framework upon which other techniques can be added. In active model learning, this trend also matches the interest in observation-tree based algorithms~\cite{SB19,vaandrager_new_2022}, which we instantiate in \CEAL. The role of observation structures in learning and testing is a long-standing lore~\cite{BGJLRS05} that can be leveraged to enhance the learning approach and its modularity with testing methods.
\section{Conclusion and Future Work}
\label{sec:conclusion}
This paper explores efficient ways to  handle conflicts during active learning. We build on the idea that recovering from
conflicts is best done by splitting information collection and the construction of the Learner's data structure, two operations that are conflated in MAT.

We introduce the Conflict Aware Active Automata Learning (\CEAL) framework as an alternative to MAT. \CEAL directly represents the SUL and introduces a \emph{Reviser} tasked with testing it, storing and curating the observations. \CEAL provides a way to accept \emph{some} conflicts to reach the Learner and to recover from them without requiring to test the SUL anew. 

To test the efficiency of \CEAL, we conducted a large body of experiments on real targets using several state-of-the-art algorithms. We found that {\bf not only does \CEAL always improves on MAT in terms of success rates}, obtaining an overall {\bf success rate of 95,5\% against MAT's 79,5\%} it most importantly {\bf enables the learning of previously un-learnable SULs}, typically complex systems plagued with a high level of noise. Our experiments further put into light the impact of equivalence tests, both in terms of variability of the results and sheer cost, with an average of {\bf 66.3\% of testing cost spent on equivalence}.

In the future, we would like to explore the use and design of testing algorithms for active learning, as their efficiency seems to be able to overshadow the difference between learning algorithms. 
\CEAL's modular nature also allows us to seamlessly build a \emph{gray-box} environment i.e. to gain information from different sources in the Reviser (e.g. specifications, access to source code). This would offset the cost of equivalence queries by using cheaper sources of observations when searching for counterexamples.

Assessing the efficiency of \CEAL on a real case of mutating targets would be of interest, as an evaluation, as an opportunity to fine-tune the framework for such task, and as a demonstration of the improved reach of active learning.
Similarly, testing the \texttt{Most Frequent} update strategy in practice against high noise levels would be of interest.

\subsection*{Acknowledgements}
This work was partially supported by the EPSRC Standard Grant CLeVer (EP/S028641/1), ERC grant Autoprobe (no. 101002697), and a Royal Society Wolfson Fellowship. The authors would like to also thank Denis Timm and the TSG and HPC groups at UCL Computer Science department for their support in running the experiments in this paper.
\vfill
\pagebreak

\bibliographystyle{plainurl}
\bibliography{biblio}

\vfill
\iftoggle{short}{}{
\pagebreak
\appendix

\section{Observation Trees}
\label{app:trees}
We now define the syntax and semantics of different observation trees, as well as their associated interface functions.

\subsection{\texttt{Most Recent}}

\begin{definition}[Observation Tree]
	A \texttt{Most Recent} observation tree is defined as a partial tree-shaped Mealy machine over alphabets $\Input$ and $\Output$ with the following structure: $\BS = (\States,q_{\varepsilon},\trans,\outf)$, as outlined in \Cref{sec:prelim}.
\end{definition}

\begin{minipage}[t]{0.4\textwidth}
    \vspace*{-0pt}\hspace*{-.5cm}
    \centering
	\begin{algorithm}[H]
        \caption{\(\Lookup_{\BS}(\mathit{in})\)}
        \label{alg:mrlookup}
        \KwData{An input word $\mathit{in}$ to lookup.}
        $\mathit{state} \gets q_{\varepsilon}$\;
        $\mathit{out} \gets \varepsilon$\;
        \For{$j \in [1,  \dots |\mathit{in}|]$}{
        	\eIf{$\delta(\mathit{state}, \mathit{in}[j]) = q_{\mathit{in}[1,j]}$}{
        		$\mathit{out} \gets \mathit{out} \cdot \lambda(\mathit{state}, \mathit{in}[j])$\;
        		$\mathit{state} \gets q_{\mathit{in}[1,j]}$\;
        	}{
        		\Return{\Null}\;
        	}
        }
        \Return{$\mathit{out}$}\;
\end{algorithm}
\raggedright
Notice that in the \texttt{Most Frequent} strategy, subtrees of \BS are suppressed in \Update in the case of conflicts. The rest of this algorithm deals with the additive construction of \BS, while \Lookup simply reads the \(\mathit{in}\) branch of \BS.
\end{minipage}
\hfill
\begin{minipage}[t]{0.525\textwidth}
    \vspace{-0pt}
    \centering
\begin{algorithm}[H]
        \caption{\(\Update_{\BS}( \mathit{in}, \mathit{out})\)}
        \label{alg:mrupdate}
        \KwData{A query pair $(\mathit{in}, \mathit{out})$ to store.}
        $\mathit{state} \gets q_{\varepsilon}$\;
        $\mathit{conflicted} \gets \bot$\;
        \For{$j \in [1,  \dots, |\mathit{in}|]$}{
        	\uIf{$\lambda( \mathit{state}, \mathit{in}[j]) = \mathit{out}[j]$}{
        		$\mathit{state} \gets \delta(\mathit{state}, \mathit{in}[j])$\;
        	}
        	\uElseIf{$\lambda( \mathit{state}, \mathit{in}[j]) = o$}{
        	Prune subtree rooted in $\delta(\mathit{state}, \mathit{in}[j])$.\\
        	$Q \gets Q \cup \{q_{\mathit{in}[1,j]}\}$\;
        	$\delta(\mathit{state}, \mathit{in}[j]) \gets q_{\mathit{in}[1,j]}$\;
        	$\lambda(\mathit{state}, \mathit{in}[j]) \gets \mathit{out}[j]$\;

        	$\mathit{conflicted} \gets \top$\;
        	$\mathit{node} \gets q_{\mathit{in}[1,j]}$\;
        	}
        	\uElse{
	        	$Q \gets Q \cup \{q_{\mathit{in}[1,j]}\}$\;
        		$\delta(\mathit{state}, \mathit{in}[j]) \gets q_{\mathit{in}[1,j]}$\;
        		$\lambda(\mathit{state}, \mathit{in}[j]) \gets \mathit{out}[j]$\;
        		$\mathit{state} \gets q_{\mathit{in}[1,j]}$\;
        	}
        }
        \Return{$\mathit{conflicted}$}\;
\end{algorithm}
\end{minipage}

\subsection{\texttt{Most Frequent}}
\begin{definition}[Observation Tree]
A \texttt{Most Frequent} observation tree is defined as a state-weighted partial tree-shaped non-deterministic Mealy machine over alphabets $\Input$ and $\Output$ with the following structure: $\BS = (\States,q_{\varepsilon},\trans, \omega)$, where $\omega : Q \to \mathbb{N}$ is the state weight function and $\trans : Q \times \Input \to 2^{Q \times \Output}$ is the transition and output function.
\end{definition}
In the case of the \texttt{Most Frequent} strategy, all observations are stored together in \BS by \Update without pruning. Conflicts arise when the most frequent output associated to an input changes, as \Lookup now chooses at each step the most frequent output for a given input, using \(\textsc{nextstate}\).

\begin{minipage}[t]{0.45\textwidth}
    \vspace*{-0pt}\hspace*{-1.25cm}
    \centering
	\begin{algorithm}[H]
        \caption{\(\Lookup_{\BS}(\mathit{in})\)}
        \label{alg:mflookup}
        \KwData{An input word $\mathit{in}$ to lookup.}
        $\mathit{state} \gets q_{\varepsilon}$\;
        $\mathit{out} \gets \varepsilon$\;
        \For{$j \in [1,  \dots, |\mathit{in}|]$}{
        $\mathit{next} \gets \textsc{nextstate}(\mathit{state}, \mathit{in}[j])$\;
        	\eIf{$\mathit{next} = (q, o)$}{
        		$\mathit{out} \gets \mathit{out} \cdot o$\;
        		$\mathit{state} \gets q$\;
        	}{
        		\Return{\Null}\;
        	}
        }
        \Return{$\mathit{out}$}\;
\end{algorithm}
\end{minipage}
\hfill
\begin{minipage}[t]{0.5375\textwidth}
    \vspace{-0pt}\hspace*{-1cm}
    \centering
\begin{algorithm}[H]
        \caption{\(\Update_{\BS}( \mathit{in}, \mathit{out})\)}
        \label{alg:mfupdate}
        \KwData{A query pair $(\mathit{in}, \mathit{out})$ to store.}
        $\mathit{state} \gets q_{\varepsilon}$\;
        $\mathit{conflicted} \gets \bot$\;
        $\mathit{mainbranch} \gets \top$\;
        \For{$j \in [1,  \dots,  |\mathit{in}|]$}{
        	\uIf{$(q, \mathit{out}[j]) \in \delta(\mathit{state}, \mathit{in}[j])$}{
                        $\mathit{temp} \gets \textsc{nextstate}_{\BS}(\mathit{state})$\;
        		$\omega(q) \gets \omega(q) + 1$\;
                        \uIf{\(\mathit{mainbranch}\ \wedge\ \mathit{temp}\neq \textsc{nextstate}_{\BS}(\mathit{state})\)}{
                        \(\mathit{conflicted}\gets \top\)\;
                        }
                        \uIf{$q\neq \mathit{temp}$}{
                        \(\mathit{mainbranch}\gets\bot\)\;
                        }
        		$\mathit{state} \gets q$\;
        	}
        	\uElse{
        		$Q \gets Q \cup \{q_{\mathit{in}[1,j]}\}$\;
        		$\delta(\mathit{state}, \mathit{in}[j]) \gets \delta(\mathit{state}, \mathit{in}[j]) \cup \{(q_{\mathit{in}[1,j]}, \mathit{out}[j])\}$\;
				$\omega(q_{\mathit{in}[1,j]}) \gets 1$\;
        		$\mathit{state} \gets q_{\mathit{in}[1,j]}$\;
        	}
        }
        \Return{$\mathit{conflicted}$}\;
\end{algorithm}
\end{minipage}

\begin{algorithm}[H]
        \caption{\(\textsc{nextstate}_{\BS}(\mathit{state}, \mathit{symbol})\)}
        \label{alg:nextchild}
        \KwData{Input $\mathit{state}$ to get next (highest counter) successor with $\mathit{symbol}$ input label.}
        $\mathit{next} \gets \Null$\;
        $\mathit{max} \gets 0$\;
        \For{$(q, o) \in \delta(\mathit{state}, \mathit{symbol})$}{
        	\If{$\omega(q) > \mathit{max}$}{
        		$\mathit{max} \gets \omega(q)$\;
        		$\mathit{next} \gets (q, o)$\;
        	}
        }
        \Return{$\mathit{next}$}\;
\end{algorithm}

\section{Proofs of Correctness}
\label{app:proof}
In this appendix, the proofs of the different properties stated in the article are made.

We separate \Cref{lem:BSandLang} into two smaller claims: 

\begin{lemma}
    \label{lem:BSUpdate}
    During an execution of \CEAL, all tests on the System are integrated in \(\BS\) through \(\Update_\BS\). These are the only modifications made to \BS.
\end{lemma}
\begin{lemma}
    \label{lem:QueryLang}
    During an execution of \CEAL, the Learner queries are answered according to \(\lang_\BS\).
\end{lemma}
This is done because \Cref{lem:BSUpdate} can be proved independently and is necessary for the proof of other properties, while the proof of \Cref{lem:QueryLang} depends on these properties.

\begin{proof}[Proof of \Cref{lem:BSUpdate}]
Calls to \System are only made in $\Read_\BS$ and $\Test_\BS$ and are directly followed by calls to \(\Update_\BS\) on their result. Thus all System tests are integrated in \BS through \(\Update_\BS\). 

As no other calls to \(\Update_\BS\) and no direct modification of \BS are performed, these are the only modifications made to \BS.
\end{proof}

\mostrecent*
\begin{proof}
We make this proof by induction on the observation stream by considering the implementations of \Lookup and \Update for the \texttt{Most Recent} update strategy.
\smallskip

By \Cref*{lem:BSUpdate} it is enough to show that 
\(
    \lang_T=
    \{(\iword_k,\oword_k)\mid 0\leq k\leq K\ \wedge\ \not\exists k<l\leq K,\text{ s.t. } (\iword_k,\oword_k)\conflicts(\iword_l,\oword_l)\}
\)
for any $K\in\bbN$ is an invariant of \(\Update_{\BS}\). 
We first note that, for the \texttt{Most Recent} update strategy, \(\Lookup_\BS\) simply returns the unique output word associated with an input word in \BS. 

Initialization: Starting with \BS the empty observation tree, all calls to \(\Lookup_\BS\) return \Null~and thus \(\lang_\BS=\emptyset\). Initially, before $K=0$, no observation has been received and the set of prefixes is empty, hence we have our result.

Induction: Suppose that before a given $K\in\bbN$ the property holds.  
When the next observation \((\iword_{k},\oword_{k})\) is processed by \(\Update_\BS\) the algorithm progresses through \((\iword_k[j],\oword_k[j])\) with the main for loop. 
As long as the considered prefix matches the structure of \(\BS\) (i.e. \(\lambda( \mathit{state}, \iword_k[j]) = \oword_k[j]\)), no modification is made and hence \(\lang_\BS\) does not change. When this does not happen,
there are then two cases to consider:
\begin{itemize}
    \item If there is no node corresponding to the transition, i.e. $\lambda( \mathit{state}, \iword_k[j])$ is undefined, then a new node is created and the transition and output functions defined according to \(\oword_k[j]\). Due to the tree structure of \BS, the rest of the word is then added to \BS with the reminder of the \textbf{for} loop. In this case, notice that \(\lang_\BS:= \lang_\BS\cup\{(\iword_k[1,l],\oword_k[1,l])\mid j\leq l \leq |\iword|\}\).
    \item $\lambda( \mathit{state}, \iword_k[j]) = \sigma\neq\oword_k[j]$. In this case, the subtree rooted in $\lambda( \mathit{state}, \iword_k[j])$ is suppressed, replaced by a new node and the transition and output functions defined according to \(\oword_k[j]\). The rest of the word is then handled according to the previous case as \BS does not have successor nodes for it. Notice that this case corresponds exactly to \((\iword_k,\oword_k)\conflicts (\iword,\oword)\) for all $(\iword,\oword)$ previously stored in the suppressed subtree, with \(\iword_k[1,j]\) the witness prefix. 
    Thus, without taking the following reconstruction of the new subtree, \(\lang_\BS:= \lang_\BS\setminus\{(\iword,\oword)\mid (\iword_k,\oword_k)\conflicts(\iword,\oword)\}\).
\end{itemize} 
By combining the two previous cases, we have that, noting \(\BS_{K-1}\) the observation tree before step $K$ and \(\BS_K\) the observation tree after step $K$: 
\[\lang_{\BS_K}=\left(\lang_{\BS_{K-1}}\setminus\{(\iword,\oword)\mid (\iword_k,\oword_k)\conflicts(\iword,\oword)\}\right)\cup\pref(\iword_k,\oword_k)\]
which proves our result.
\end{proof}

\mostfrequent*
\begin{proof}
We make this proof by induction on the observation stream by considering the implementations of \Lookup and \Update for the \texttt{Most Recent} update strategy.
\smallskip

By \Cref*{lem:BSUpdate} it is enough to show that 
\(
    \lang_T=
    \{\pref(\iword_k,\oword_k)\mid k\leq K\ \wedge\ \not\exists <l\leq K,\text{ s.t. } (\iword_k,\oword_k)\conflicts(\iword_l,\oword_l)\ \wedge \mathit{mf}((\iword_k,\oword_k), (\iword_l,\oword_l))\}
\)
for any $K\in\bbN$ is an invariant of \(\Update_{\BS}\). 
Notice that for this update strategy, $\Lookup_\BS(in)$ has to make a choice in the subtree with input word $in$. It does so by iteratively choosing the output that has been seen the most often at each step.

Initialization: before any observation has been received, \(\BS\) is empty, hence \(\lang_\BS=\emptyset\). Similarly, no observation has been received yet, thus the set of prefixes is empty and we have our result.

Induction: Suppose that before a given $K\in\bbN$ the property holds.  
When the next observation \((\iword_{k},\oword_{k})\) is processed by \(\Update_\BS\) the algorithm progresses through \((\iword_k[j],\oword_k[j])\) with the main for loop. As long as the considered prefix corresponds to what \(\Lookup_\BS\) would read, i.e. \(\oword_k[j]\) is the output with most weight associated to \(\iword_k[j]\), the only modification is the increment of the weight of the states visited, which already are the maximal weight choice, thus \(\lang_\BS\) is not modified. 
Notice that when the branch of maximal weight is left, \(\mathit{mainbranch}\) is set to false, preventing conflicts to be raised. This corresponds to the fact that \(\Lookup_\BS[\iword_k[1,j]]\) would not reach the part of the tree that will be subsequently explored. Then, as long as the target state and transition exist, the only remaining difference is the increment of the weights.
There are two other cases to consider:
\begin{itemize}
    \item If there is no node corresponding to the transition, i.e. for any node $q$, $(q,\oword_k[j])\notin\delta( \mathit{state}, \iword_k[j])$, then a new node is created and the transition and output functions defined according to \(\oword_k[j]\). Due to the tree structure of \BS, the rest of the word is then added to \BS with the reminder of the \textbf{for} loop. In this case, notice that \(\lang_\BS:= \lang_\BS\cup\{(\iword_k[1,l],\oword_k[1,l])\mid j\leq l \leq |\iword|\}\) \emph{only if} the main branch was not left before.
    \item If the main branch is not left (i.e. if this part of the tree is reached by calls to \Lookup) and the state of maximal weight appearing in \(\delta( \mathit{state}, \iword_k[j])\) changes due to the increment, the subtree reached by \Lookup changes based on the new maximum. With a storage of results in \(\delta\) favoring most recent observations, this corresponds exactly to \((\iword_k,\oword_k)\conflicts(\iword,\oword)\ \wedge \mathit{mf}((\iword_k,\oword_k), (\iword,\oword))\) for all \((\iword,\oword)\) in the previously reached subtree. Then the correct increment of the weights at the previous steps allows to conclude.
\end{itemize} 
The previous cases guarantee that weights are correctly incremented and that \Lookup will select its results according to these weights, granting us our result. 
\end{proof}

\prune*
\begin{proof}
We conduct this proof in two parts: first, we show that \(\Update_\BS\) sends a \Restart signal exactly when the last observation conflicts with \BS.; second, we add that the algorithms of \CEAL send \Restart exactly when \Update returns it.

\smallskip
Remark, following the proofs of \Cref{prop:inv_rec} and \Cref{prop:inv_freq}, that \(\top\) is only returned when a non-additive change to \(\lang_\BS\) is performed. This is trivial for the \texttt{Most Recent} strategy, and is ensured for the \texttt{Most Frequent} by the boolean \(\mathit{mainbranch}\) stopping changes to \(\mathsf{nextstate}\) outside of the part of the tree reached by \(\Lookup_\BS\) to return \(\top\).

Of the \CEAL algorithms, only the two algorithms calling \Apply, and through it \(\Update\) can return a \Restart signal to the Learner. These do so exactly when \(\Apply\) returns \(\Restart\) i.e. when \(\Update_\BS\) returns \(\top\).
\end{proof}

\BSandLang*
\begin{proof}
    Corollary of \Cref{lem:BSUpdate} and \Cref{lem:QueryLang}.
\end{proof}
\section{Detailed Experimental Results}
\label{app:results}
The following table lists the best learning configurations (algorithm and number of repeats) for each frameworks and each learning experiment (target, type and level of noise). 

Each experiment is performed 100 times and times-out after 10 million queries on the SUL. The success rate is the number of successful runs and the test count is the average number of tests during a successful run.
\begin{longtable}[c]{@{}llllrr@{}}
\toprule
\multicolumn{1}{c}{Experiment} & \multicolumn{1}{c}{Framework} & \multicolumn{1}{c}{Algorithm} & \multicolumn{1}{c}{Repeats} & \multicolumn{1}{c}{Success Rate} & Test Count \\* \midrule
\endfirsthead
\multicolumn{6}{c}%
{{\bfseries Table \thetable\ continued from previous page}} \\
\toprule
\multicolumn{1}{c}{Experiment} & \multicolumn{1}{c}{Framework} & \multicolumn{1}{c}{Algorithm} & \multicolumn{1}{c}{Repeats} & \multicolumn{1}{c}{Success Rate} & Test Count \\* \midrule
\endhead
\multirow{2}{*}{\begin{tabular}[c]{@{}l@{}}4\_learnresult\_PIN\_fix.dot\\ INPUT noise at 0.01\%\end{tabular}} & MAT & TTT & (5,10) & 100 & 18800.07 \\* \cmidrule(l){2-6} 
 & \CEAL & RS & (5,10) & 100 & 23667.96 \\* \midrule
\multirow{2}{*}{\begin{tabular}[c]{@{}l@{}}4\_learnresult\_PIN\_fix.dot\\ OUTPUT noise at 0.01\%\end{tabular}} & MAT & TTT & (5,10) & 100 & 20012.76 \\* \cmidrule(l){2-6} 
 & \CEAL & RS & (5,10) & 100 & 23650.55 \\* \midrule
\multirow{2}{*}{\begin{tabular}[c]{@{}l@{}}4\_learnresult\_PIN\_fix.dot\\ INPUT noise at 0.05\%\end{tabular}} & \CEAL & RS & (5,10) & 100 & 25818.2 \\* \cmidrule(l){2-6} 
 & MAT & KV & (20,30) & 91 & 106431.74 \\* \midrule
\multirow{2}{*}{\begin{tabular}[c]{@{}l@{}}4\_learnresult\_PIN\_fix.dot\\ OUTPUT noise at 0.05\%\end{tabular}} & \CEAL & RS & (5,10) & 100 & 25141.6 \\* \cmidrule(l){2-6} 
 & MAT & RS & (20,30) & 91 & 127966.04 \\* \midrule
\multirow{2}{*}{\begin{tabular}[c]{@{}l@{}}4\_learnresult\_PIN\_fix.dot\\ INPUT noise at 0.1\%\end{tabular}} & \CEAL & RS & (10,20) & 100 & 57116.14 \\* \cmidrule(l){2-6} 
 & MAT & KV & (20,30) & 20 & 53483.85 \\* \midrule
\multirow{2}{*}{\begin{tabular}[c]{@{}l@{}}4\_learnresult\_PIN\_fix.dot\\ OUTPUT noise at 0.1\%\end{tabular}} & \CEAL & RS & (5,10) & 100 & 35095.97 \\* \cmidrule(l){2-6} 
 & MAT & TTT & (20,30) & 18 & 36669.28 \\* \midrule
\multirow{2}{*}{\begin{tabular}[c]{@{}l@{}}ASN\_learnresult\_MAESTRO\_fix.dot\\ INPUT noise at 0.01\%\end{tabular}} & MAT & TTT & (5,10) & 100 & 19454.34 \\* \cmidrule(l){2-6} 
 & \CEAL & RS & (5,10) & 100 & 23669.49 \\* \midrule
\multirow{2}{*}{\begin{tabular}[c]{@{}l@{}}ASN\_learnresult\_MAESTRO\_fix.dot\\ OUTPUT noise at 0.01\%\end{tabular}} & MAT & TTT & (5,10) & 100 & 20906.17 \\* \cmidrule(l){2-6} 
 & \CEAL & RS & (5,10) & 100 & 23651.8 \\* \midrule
\multirow{2}{*}{\begin{tabular}[c]{@{}l@{}}ASN\_learnresult\_MAESTRO\_fix.dot\\ INPUT noise at 0.05\%\end{tabular}} & \CEAL & RS & (5,10) & 100 & 25683.77 \\* \cmidrule(l){2-6} 
 & MAT & TTT & (20,30) & 91 & 77650.24 \\* \midrule
\multirow{2}{*}{\begin{tabular}[c]{@{}l@{}}ASN\_learnresult\_MAESTRO\_fix.dot\\ OUTPUT noise at 0.05\%\end{tabular}} & \CEAL & RS & (5,10) & 100 & 25294.13 \\* \cmidrule(l){2-6} 
 & MAT & RS & (20,30) & 96 & 109329.89 \\* \midrule
\multirow{2}{*}{\begin{tabular}[c]{@{}l@{}}ASN\_learnresult\_MAESTRO\_fix.dot\\ INPUT noise at 0.1\%\end{tabular}} & \CEAL & RS & (10,20) & 100 & 56806.28 \\* \cmidrule(l){2-6} 
 & MAT & RS & (20,30) & 26 & 98705.58 \\* \midrule
\multirow{2}{*}{\begin{tabular}[c]{@{}l@{}}ASN\_learnresult\_MAESTRO\_fix.dot\\ OUTPUT noise at 0.1\%\end{tabular}} & \CEAL & RS & (5,10) & 100 & 33890 \\* \cmidrule(l){2-6} 
 & MAT & KV & (20,30) & 19 & 53137.53 \\* \midrule
\multirow{2}{*}{\begin{tabular}[c]{@{}l@{}}Rabo\_learnresult\_MAESTRO\_fix.dot\\ INPUT noise at 0.01\%\end{tabular}} & MAT & TTT & (5,10) & 100 & 18143.9 \\* \cmidrule(l){2-6} 
 & \CEAL & RS & (5,10) & 100 & 23661.77 \\* \midrule
\multirow{2}{*}{\begin{tabular}[c]{@{}l@{}}Rabo\_learnresult\_MAESTRO\_fix.dot\\ OUTPUT noise at 0.01\%\end{tabular}} & MAT & TTT & (5,10) & 100 & 20639.11 \\* \cmidrule(l){2-6} 
 & \CEAL & RS & (5,10) & 100 & 23652.36 \\* \midrule
\multirow{2}{*}{\begin{tabular}[c]{@{}l@{}}Rabo\_learnresult\_MAESTRO\_fix.dot\\ INPUT noise at 0.05\%\end{tabular}} & \CEAL & RS & (5,10) & 100 & 25762.24 \\* \cmidrule(l){2-6} 
 & MAT & TTT & (20,30) & 93 & 78030.14 \\* \midrule
\multirow{2}{*}{\begin{tabular}[c]{@{}l@{}}Rabo\_learnresult\_MAESTRO\_fix.dot\\ OUTPUT noise at 0.05\%\end{tabular}} & \CEAL & RS & (5,10) & 100 & 25272.71 \\* \cmidrule(l){2-6} 
 & MAT & RS & (20,30) & 93 & 155886.7 \\* \midrule
\multirow{2}{*}{\begin{tabular}[c]{@{}l@{}}Rabo\_learnresult\_MAESTRO\_fix.dot\\ INPUT noise at 0.1\%\end{tabular}} & \CEAL & RS & (5,10) & 100 & 56333.1 \\* \cmidrule(l){2-6} 
 & MAT & RS & (20,30) & 25 & 65872.72 \\* \midrule
\multirow{2}{*}{\begin{tabular}[c]{@{}l@{}}Rabo\_learnresult\_MAESTRO\_fix.dot\\ OUTPUT noise at 0.1\%\end{tabular}} & \CEAL & RS & (5,10) & 100 & 34832.46 \\* \cmidrule(l){2-6} 
 & MAT & RS & (20,30) & 24 & 79005.62 \\* \midrule
\multirow{2}{*}{\begin{tabular}[c]{@{}l@{}}Rabo\_learnresult\_SecureCode\_Aut\_fix.dot\\ INPUT noise at 0.01\%\end{tabular}} & MAT & TTT & (5,10) & 100 & 13191.67 \\* \cmidrule(l){2-6} 
 & \CEAL & TTT & (5,10) & 100 & 28092.53 \\* \midrule
\multirow{2}{*}{\begin{tabular}[c]{@{}l@{}}Rabo\_learnresult\_SecureCode\_Aut\_fix.dot\\ OUTPUT noise at 0.01\%\end{tabular}} & MAT & TTT & (5,10) & 100 & 11904.83 \\* \cmidrule(l){2-6} 
 & \CEAL & TTT & (5,10) & 100 & 27709.41 \\* \midrule
\multirow{2}{*}{\begin{tabular}[c]{@{}l@{}}Rabo\_learnresult\_SecureCode\_Aut\_fix.dot\\ INPUT noise at 0.05\%\end{tabular}} & \CEAL & RS & (5,10) & 100 & 31835.22 \\* \cmidrule(l){2-6} 
 & MAT & KV & (20,30) & 95 & 80110.18 \\* \midrule
\multirow{2}{*}{\begin{tabular}[c]{@{}l@{}}Rabo\_learnresult\_SecureCode\_Aut\_fix.dot\\ OUTPUT noise at 0.05\%\end{tabular}} & \CEAL & RS & (5,10) & 100 & 30321.21 \\* \cmidrule(l){2-6} 
 & MAT & KV & (20,30) & 94 & 81503.62 \\* \midrule
\multirow{2}{*}{\begin{tabular}[c]{@{}l@{}}Rabo\_learnresult\_SecureCode\_Aut\_fix.dot\\ INPUT noise at 0.1\%\end{tabular}} & \CEAL & RS & (10,20) & 100 & 71161.74 \\* \cmidrule(l){2-6} 
 & MAT & RS & (20,30) & 40 & 68508.25 \\* \midrule
\multirow{2}{*}{\begin{tabular}[c]{@{}l@{}}Rabo\_learnresult\_SecureCode\_Aut\_fix.dot\\ OUTPUT noise at 0.1\%\end{tabular}} & \CEAL & RS & (5,10) & 100 & 43735.72 \\* \cmidrule(l){2-6} 
 & MAT & RS & (20,30) & 34 & 68287.74 \\* \midrule
\multirow{2}{*}{\begin{tabular}[c]{@{}l@{}}GnuTLS\_3.3.12\_client\_regular.dot\\ INPUT noise at 0.01\%\end{tabular}} & \CEAL & RS & (5,10) & 100 & 8700.15 \\* \cmidrule(l){2-6} 
 & MAT & KV & (5,10) & 100 & 13312.08 \\* \midrule
\multirow{2}{*}{\begin{tabular}[c]{@{}l@{}}GnuTLS\_3.3.12\_client\_regular.dot\\ OUTPUT noise at 0.01\%\end{tabular}} & \CEAL & RS & (5,10) & 100 & 8714.03 \\* \cmidrule(l){2-6} 
 & MAT & KV & (5,10) & 100 & 14611.65 \\* \midrule
\multirow{2}{*}{\begin{tabular}[c]{@{}l@{}}GnuTLS\_3.3.12\_client\_regular.dot\\ INPUT noise at 0.05\%\end{tabular}} & \CEAL & RS & (5,10) & 100 & 9383.59 \\* \cmidrule(l){2-6} 
 & MAT & KV & (5,10) & 100 & 13765.59 \\* \midrule
\multirow{2}{*}{\begin{tabular}[c]{@{}l@{}}GnuTLS\_3.3.12\_client\_regular.dot\\ OUTPUT noise at 0.05\%\end{tabular}} & \CEAL & RS & (5,10) & 100 & 9542.08 \\* \cmidrule(l){2-6} 
 & MAT & KV & (10,20) & 100 & 31961.76 \\* \midrule
\multirow{2}{*}{\begin{tabular}[c]{@{}l@{}}GnuTLS\_3.3.12\_client\_regular.dot\\ INPUT noise at 0.1\%\end{tabular}} & \CEAL & RS & (5,10) & 100 & 16366.54 \\* \cmidrule(l){2-6} 
 & MAT & KV & (20,30) & 100 & 56634.7 \\* \midrule
\multirow{2}{*}{\begin{tabular}[c]{@{}l@{}}GnuTLS\_3.3.12\_client\_regular.dot\\ OUTPUT noise at 0.1\%\end{tabular}} & \CEAL & RS & (5,10) & 100 & 13181.22 \\* \cmidrule(l){2-6} 
 & MAT & RS & (20,30) & 99 & 108314.29 \\* \midrule
\multirow{2}{*}{\begin{tabular}[c]{@{}l@{}}GnuTLS\_3.3.12\_server\_regular.dot\\ INPUT noise at 0.01\%\end{tabular}} & \CEAL & RS & (5,10) & 100 & 8700.66 \\* \cmidrule(l){2-6} 
 & MAT & KV & (5,10) & 100 & 12742.74 \\* \midrule
\multirow{2}{*}{\begin{tabular}[c]{@{}l@{}}GnuTLS\_3.3.12\_server\_regular.dot\\ OUTPUT noise at 0.01\%\end{tabular}} & \CEAL & RS & (5,10) & 100 & 8716.5 \\* \cmidrule(l){2-6} 
 & MAT & KV & (5,10) & 100 & 12744.21 \\* \midrule
\multirow{2}{*}{\begin{tabular}[c]{@{}l@{}}GnuTLS\_3.3.12\_server\_regular.dot\\ INPUT noise at 0.05\%\end{tabular}} & \CEAL & RS & (5,10) & 100 & 9238.11 \\* \cmidrule(l){2-6} 
 & MAT & RS & (5,10) & 100 & 21829.48 \\* \midrule
\multirow{2}{*}{\begin{tabular}[c]{@{}l@{}}GnuTLS\_3.3.12\_server\_regular.dot\\ OUTPUT noise at 0.05\%\end{tabular}} & \CEAL & RS & (5,10) & 100 & 9579.21 \\* \cmidrule(l){2-6} 
 & MAT & KV & (10,20) & 100 & 23058.97 \\* \midrule
\multirow{2}{*}{\begin{tabular}[c]{@{}l@{}}GnuTLS\_3.3.12\_server\_regular.dot\\ INPUT noise at 0.1\%\end{tabular}} & \CEAL & RS & (5,10) & 100 & 16750.42 \\* \cmidrule(l){2-6} 
 & MAT & KV & (20,30) & 100 & 52193.66 \\* \midrule
\multirow{2}{*}{\begin{tabular}[c]{@{}l@{}}GnuTLS\_3.3.12\_server\_regular.dot\\ OUTPUT noise at 0.1\%\end{tabular}} & \CEAL & RS & (5,10) & 100 & 13136.49 \\* \cmidrule(l){2-6} 
 & MAT & RS & (20,30) & 100 & 109055.26 \\* \midrule
\multirow{2}{*}{\begin{tabular}[c]{@{}l@{}}NSS\_3.17.4\_client\_regular.dot\\ INPUT noise at 0.01\%\end{tabular}} & \CEAL & RS & (5,10) & 100 & 9678.38 \\* \cmidrule(l){2-6} 
 & MAT & KV & (5,10) & 100 & 12744.42 \\* \midrule
\multirow{2}{*}{\begin{tabular}[c]{@{}l@{}}NSS\_3.17.4\_client\_regular.dot\\ OUTPUT noise at 0.01\%\end{tabular}} & \CEAL & RS & (5,10) & 100 & 9684.25 \\* \cmidrule(l){2-6} 
 & MAT & TTT & (5,10) & 100 & 12932.33 \\* \midrule
\multirow{2}{*}{\begin{tabular}[c]{@{}l@{}}NSS\_3.17.4\_client\_regular.dot\\ INPUT noise at 0.05\%\end{tabular}} & \CEAL & RS & (5,10) & 100 & 10643.94 \\* \cmidrule(l){2-6} 
 & MAT & TTT & (5,10) & 100 & 14856.42 \\* \midrule
\multirow{2}{*}{\begin{tabular}[c]{@{}l@{}}NSS\_3.17.4\_client\_regular.dot\\ OUTPUT noise at 0.05\%\end{tabular}} & \CEAL & RS & (5,10) & 100 & 10755.87 \\* \cmidrule(l){2-6} 
 & MAT & TTT & (5,10) & 100 & 14001.56 \\* \midrule
\multirow{2}{*}{\begin{tabular}[c]{@{}l@{}}NSS\_3.17.4\_client\_regular.dot\\ INPUT noise at 0.1\%\end{tabular}} & \CEAL & RS & (10,20) & 100 & 24055.3 \\* \cmidrule(l){2-6} 
 & MAT & KV & (10,20) & 100 & 27046.64 \\* \midrule
\multirow{2}{*}{\begin{tabular}[c]{@{}l@{}}NSS\_3.17.4\_client\_regular.dot\\ OUTPUT noise at 0.1\%\end{tabular}} & \CEAL & RS & (5,10) & 100 & 16435.53 \\* \cmidrule(l){2-6} 
 & MAT & TTT & (10,20) & 100 & 31956.55 \\* \midrule
\multirow{2}{*}{\begin{tabular}[c]{@{}l@{}}OpenSSL\_1.0.2\_server\_regular.dot\\ INPUT noise at 0.01\%\end{tabular}} & \CEAL & RS & (5,10) & 100 & 6417.43 \\* \cmidrule(l){2-6} 
 & MAT & KV & (5,10) & 100 & 10038.55 \\* \midrule
\multirow{2}{*}{\begin{tabular}[c]{@{}l@{}}OpenSSL\_1.0.2\_server\_regular.dot\\ OUTPUT noise at 0.01\%\end{tabular}} & \CEAL & RS & (5,10) & 100 & 6434.04 \\* \cmidrule(l){2-6} 
 & MAT & KV & (5,10) & 100 & 9668.83 \\* \midrule
\multirow{2}{*}{\begin{tabular}[c]{@{}l@{}}OpenSSL\_1.0.2\_server\_regular.dot\\ INPUT noise at 0.05\%\end{tabular}} & \CEAL & RS & (5,10) & 100 & 6825.86 \\* \cmidrule(l){2-6} 
 & MAT & KV & (10,20) & 100 & 19902.09 \\* \midrule
\multirow{2}{*}{\begin{tabular}[c]{@{}l@{}}OpenSSL\_1.0.2\_server\_regular.dot\\ OUTPUT noise at 0.05\%\end{tabular}} & \CEAL & RS & (5,10) & 100 & 7066.48 \\* \cmidrule(l){2-6} 
 & MAT & KV & (5,10) & 100 & 10986.84 \\* \midrule
\multirow{2}{*}{\begin{tabular}[c]{@{}l@{}}OpenSSL\_1.0.2\_server\_regular.dot\\ INPUT noise at 0.1\%\end{tabular}} & \CEAL & RS & (5,10) & 100 & 10072.76 \\* \cmidrule(l){2-6} 
 & MAT & TTT & (10,20) & 100 & 24247.78 \\* \midrule
\multirow{2}{*}{\begin{tabular}[c]{@{}l@{}}OpenSSL\_1.0.2\_server\_regular.dot\\ OUTPUT noise at 0.1\%\end{tabular}} & \CEAL & RS & (5,10) & 100 & 8997.39 \\* \cmidrule(l){2-6} 
 & MAT & KV & (20,30) & 100 & 51616.12 \\* \midrule
\multirow{2}{*}{\begin{tabular}[c]{@{}l@{}}Volksbank\_learnresult\_MAESTRO\_fix.dot\\ INPUT noise at 0.01\%\end{tabular}} & MAT & TTT & (5,10) & 100 & 7943.4 \\* \cmidrule(l){2-6} 
 & \CEAL & TTT & (5,10) & 100 & 17237.49 \\* \midrule
\multirow{2}{*}{\begin{tabular}[c]{@{}l@{}}Volksbank\_learnresult\_MAESTRO\_fix.dot\\ OUTPUT noise at 0.01\%\end{tabular}} & MAT & TTT & (5,10) & 100 & 8030.57 \\* \cmidrule(l){2-6} 
 & \CEAL & TTT & (5,10) & 100 & 17367.51 \\* \midrule
\multirow{2}{*}{\begin{tabular}[c]{@{}l@{}}Volksbank\_learnresult\_MAESTRO\_fix.dot\\ INPUT noise at 0.05\%\end{tabular}} & \CEAL & RS & (5,10) & 100 & 34281.38 \\* \cmidrule(l){2-6} 
 & MAT & RS & (20,30) & 99 & 63789.12 \\* \midrule
\multirow{2}{*}{\begin{tabular}[c]{@{}l@{}}Volksbank\_learnresult\_MAESTRO\_fix.dot\\ OUTPUT noise at 0.05\%\end{tabular}} & \CEAL & RS & (5,10) & 100 & 32941.97 \\* \cmidrule(l){2-6} 
 & MAT & RS & (20,30) & 100 & 56836.2 \\* \midrule
\multirow{2}{*}{\begin{tabular}[c]{@{}l@{}}Volksbank\_learnresult\_MAESTRO\_fix.dot\\ INPUT noise at 0.1\%\end{tabular}} & \CEAL & RS & (10,20) & 100 & 80374.4 \\* \cmidrule(l){2-6} 
 & MAT & RS & (20,30) & 68 & 55511.59 \\* \midrule
\multirow{2}{*}{\begin{tabular}[c]{@{}l@{}}Volksbank\_learnresult\_MAESTRO\_fix.dot\\ OUTPUT noise at 0.1\%\end{tabular}} & \CEAL & RS & (5,10) & 100 & 51674.12 \\* \cmidrule(l){2-6} 
 & MAT & RS & (20,30) & 71 & 55008.2 \\* \midrule
\multirow{2}{*}{\begin{tabular}[c]{@{}l@{}}NSS\_3.17.4\_server\_regular.dot\\ INPUT noise at 0.01\%\end{tabular}} & \CEAL & RS & (5,10) & 100 & 10638.69 \\* \cmidrule(l){2-6} 
 & MAT & TTT & (5,10) & 100 & 12000.78 \\* \midrule
\multirow{2}{*}{\begin{tabular}[c]{@{}l@{}}NSS\_3.17.4\_server\_regular.dot\\ OUTPUT noise at 0.01\%\end{tabular}} & \CEAL & RS & (5,10) & 100 & 10647.57 \\* \cmidrule(l){2-6} 
 & MAT & TTT & (5,10) & 100 & 11445.93 \\* \midrule
\multirow{2}{*}{\begin{tabular}[c]{@{}l@{}}NSS\_3.17.4\_server\_regular.dot\\ INPUT noise at 0.05\%\end{tabular}} & \CEAL & RS & (5,10) & 100 & 11705.27 \\* \cmidrule(l){2-6} 
 & MAT & TTT & (10,20) & 100 & 26472.28 \\* \midrule
\multirow{2}{*}{\begin{tabular}[c]{@{}l@{}}NSS\_3.17.4\_server\_regular.dot\\ OUTPUT noise at 0.05\%\end{tabular}} & \CEAL & RS & (5,10) & 100 & 11767.46 \\* \cmidrule(l){2-6} 
 & MAT & TTT & (5,10) & 100 & 12457.79 \\* \midrule
\multirow{2}{*}{\begin{tabular}[c]{@{}l@{}}NSS\_3.17.4\_server\_regular.dot\\ INPUT noise at 0.1\%\end{tabular}} & \CEAL & RS & (10,20) & 100 & 25801.2 \\* \cmidrule(l){2-6} 
 & MAT & KV & (10,20) & 100 & 31439.29 \\* \midrule
\multirow{2}{*}{\begin{tabular}[c]{@{}l@{}}NSS\_3.17.4\_server\_regular.dot\\ OUTPUT noise at 0.1\%\end{tabular}} & \CEAL & RS & (5,10) & 100 & 17591.27 \\* \cmidrule(l){2-6} 
 & MAT & TTT & (20,30) & 100 & 59005.09 \\* \midrule
\multirow{2}{*}{\begin{tabular}[c]{@{}l@{}}GnuTLS\_3.3.12\_client\_full.dot\\ INPUT noise at 0.01\%\end{tabular}} & MAT & TTT & (5,10) & 100 & 91160.43 \\* \cmidrule(l){2-6} 
 & \CEAL & KV & (5,10) & 100 & 205658.94 \\* \midrule
\multirow{2}{*}{\begin{tabular}[c]{@{}l@{}}GnuTLS\_3.3.12\_client\_full.dot\\ OUTPUT noise at 0.01\%\end{tabular}} & MAT & TTT & (5,10) & 100 & 113987.9 \\* \cmidrule(l){2-6} 
 & \CEAL & TTT & (5,10) & 100 & 194448.25 \\* \midrule
\multirow{2}{*}{\begin{tabular}[c]{@{}l@{}}GnuTLS\_3.3.12\_client\_full.dot\\ INPUT noise at 0.05\%\end{tabular}} & \CEAL & TTT & (5,10) & 100 & 201814.9 \\* \cmidrule(l){2-6} 
 & MAT & TTT & (10,20) & 100 & 227056.17 \\* \midrule
\multirow{2}{*}{\begin{tabular}[c]{@{}l@{}}GnuTLS\_3.3.12\_client\_full.dot\\ OUTPUT noise at 0.05\%\end{tabular}} & \CEAL & TTT & (5,10) & 100 & 228027.96 \\* \cmidrule(l){2-6} 
 & MAT & KV & (10,20) & 100 & 247702.66 \\* \midrule
\multirow{2}{*}{\begin{tabular}[c]{@{}l@{}}GnuTLS\_3.3.12\_client\_full.dot\\ INPUT noise at 0.1\%\end{tabular}} & MAT & KV & (20,30) & 100 & 415169.16 \\* \cmidrule(l){2-6} 
 & \CEAL & TTT & (10,20) & 100 & 466326.9 \\* \midrule
\multirow{2}{*}{\begin{tabular}[c]{@{}l@{}}GnuTLS\_3.3.12\_client\_full.dot\\ OUTPUT noise at 0.1\%\end{tabular}} & \CEAL & TTT & (10,20) & 100 & 550787.09 \\* \cmidrule(l){2-6} 
 & MAT & KV & (20,30) & 98 & 555802.05 \\* \midrule
\multirow{2}{*}{\begin{tabular}[c]{@{}l@{}}GnuTLS\_3.3.12\_server\_full.dot\\ INPUT noise at 0.01\%\end{tabular}} & MAT & KV & (5,10) & 100 & 36121.47 \\* \cmidrule(l){2-6} 
 & \CEAL & TTT & (5,10) & 100 & 73974.93 \\* \midrule
\multirow{2}{*}{\begin{tabular}[c]{@{}l@{}}GnuTLS\_3.3.12\_server\_full.dot\\ OUTPUT noise at 0.01\%\end{tabular}} & MAT & KV & (5,10) & 100 & 33031.4 \\* \cmidrule(l){2-6} 
 & \CEAL & TTT & (5,10) & 100 & 54344.32 \\* \midrule
\multirow{2}{*}{\begin{tabular}[c]{@{}l@{}}GnuTLS\_3.3.12\_server\_full.dot\\ INPUT noise at 0.05\%\end{tabular}} & \CEAL & KV & (5,10) & 100 & 61347.94 \\* \cmidrule(l){2-6} 
 & MAT & TTT & (10,20) & 100 & 71734.43 \\* \midrule
\multirow{2}{*}{\begin{tabular}[c]{@{}l@{}}GnuTLS\_3.3.12\_server\_full.dot\\ OUTPUT noise at 0.05\%\end{tabular}} & \CEAL & LSHARP & (5,10) & 100 & 64307.09 \\* \cmidrule(l){2-6} 
 & MAT & TTT & (10,20) & 100 & 68059.76 \\* \midrule
\multirow{2}{*}{\begin{tabular}[c]{@{}l@{}}GnuTLS\_3.3.12\_server\_full.dot\\ INPUT noise at 0.1\%\end{tabular}} & \CEAL & TTT & (5,10) & 100 & 64513.07 \\* \cmidrule(l){2-6} 
 & MAT & TTT & (20,30) & 100 & 138291 \\* \midrule
\multirow{2}{*}{\begin{tabular}[c]{@{}l@{}}GnuTLS\_3.3.12\_server\_full.dot\\ OUTPUT noise at 0.1\%\end{tabular}} & \CEAL & LSHARP & (5,10) & 100 & 48661.21 \\* \cmidrule(l){2-6} 
 & MAT & LSHARP & (20,30) & 100 & 178243.28 \\* \midrule
\multirow{2}{*}{\begin{tabular}[c]{@{}l@{}}learnresult\_fix.dot\\ INPUT noise at 0.01\%\end{tabular}} & MAT & TTT & (5,10) & 100 & 13561.85 \\* \cmidrule(l){2-6} 
 & \CEAL & TTT & (5,10) & 100 & 25606.65 \\* \midrule
\multirow{2}{*}{\begin{tabular}[c]{@{}l@{}}learnresult\_fix.dot\\ OUTPUT noise at 0.01\%\end{tabular}} & MAT & TTT & (5,10) & 100 & 13682.06 \\* \cmidrule(l){2-6} 
 & \CEAL & TTT & (5,10) & 100 & 25727.23 \\* \midrule
\multirow{2}{*}{\begin{tabular}[c]{@{}l@{}}learnresult\_fix.dot\\ INPUT noise at 0.05\%\end{tabular}} & \CEAL & TTT & (5,10) & 100 & 54883.73 \\* \cmidrule(l){2-6} 
 & MAT & TTT & (20,30) & 100 & 57206.25 \\* \midrule
\multirow{2}{*}{\begin{tabular}[c]{@{}l@{}}learnresult\_fix.dot\\ OUTPUT noise at 0.05\%\end{tabular}} & \CEAL & RS & (5,10) & 100 & 50620.84 \\* \cmidrule(l){2-6} 
 & MAT & RS & (10,20) & 99 & 49926.49 \\* \midrule
\multirow{2}{*}{\begin{tabular}[c]{@{}l@{}}learnresult\_fix.dot\\ INPUT noise at 0.1\%\end{tabular}} & \CEAL & RS & (10,20) & 100 & 121189.06 \\* \cmidrule(l){2-6} 
 & MAT & RS & (20,30) & 48 & 82305.83 \\* \midrule
\multirow{2}{*}{\begin{tabular}[c]{@{}l@{}}learnresult\_fix.dot\\ OUTPUT noise at 0.1\%\end{tabular}} & \CEAL & RS & (5,10) & 100 & 85413.46 \\* \cmidrule(l){2-6} 
 & MAT & RS & (20,30) & 74 & 94502.73 \\* \midrule
\multirow{2}{*}{\begin{tabular}[c]{@{}l@{}}OpenSSL\_1.0.2\_client\_full.dot\\ INPUT noise at 0.01\%\end{tabular}} & MAT & TTT & (5,10) & 100 & 68973.07 \\* \cmidrule(l){2-6} 
 & \CEAL & TTT & (5,10) & 100 & 135919.55 \\* \midrule
\multirow{2}{*}{\begin{tabular}[c]{@{}l@{}}OpenSSL\_1.0.2\_client\_full.dot\\ OUTPUT noise at 0.01\%\end{tabular}} & MAT & TTT & (5,10) & 100 & 65188.04 \\* \cmidrule(l){2-6} 
 & \CEAL & KV & (5,10) & 100 & 117615.6 \\* \midrule
\multirow{2}{*}{\begin{tabular}[c]{@{}l@{}}OpenSSL\_1.0.2\_client\_full.dot\\ INPUT noise at 0.05\%\end{tabular}} & \CEAL & TTT & (5,10) & 100 & 107252.09 \\* \cmidrule(l){2-6} 
 & MAT & KV & (10,20) & 100 & 130950.71 \\* \midrule
\multirow{2}{*}{\begin{tabular}[c]{@{}l@{}}OpenSSL\_1.0.2\_client\_full.dot\\ OUTPUT noise at 0.05\%\end{tabular}} & MAT & KV & (10,20) & 100 & 141440.98 \\* \cmidrule(l){2-6} 
 & \CEAL & KV & (5,10) & 100 & 149671.51 \\* \midrule
\multirow{2}{*}{\begin{tabular}[c]{@{}l@{}}OpenSSL\_1.0.2\_client\_full.dot\\ INPUT noise at 0.1\%\end{tabular}} & \CEAL & KV & (10,20) & 100 & 219337.62 \\* \cmidrule(l){2-6} 
 & MAT & TTT & (20,30) & 100 & 301707.18 \\* \midrule
\multirow{2}{*}{\begin{tabular}[c]{@{}l@{}}OpenSSL\_1.0.2\_client\_full.dot\\ OUTPUT noise at 0.1\%\end{tabular}} & MAT & KV & (20,30) & 100 & 298232.72 \\* \cmidrule(l){2-6} 
 & \CEAL & TTT & (10,20) & 100 & 415644.59 \\* \midrule
\multirow{2}{*}{\begin{tabular}[c]{@{}l@{}}RSA\_BSAFE\_C\_4.0.4\_server\_regular.dot\\ INPUT noise at 0.01\%\end{tabular}} & \CEAL & RS & (5,10) & 100 & 11271.48 \\* \cmidrule(l){2-6} 
 & MAT & TTT & (5,10) & 100 & 14535.15 \\* \midrule
\multirow{2}{*}{\begin{tabular}[c]{@{}l@{}}RSA\_BSAFE\_C\_4.0.4\_server\_regular.dot\\ OUTPUT noise at 0.01\%\end{tabular}} & \CEAL & RS & (5,10) & 100 & 11274.89 \\* \cmidrule(l){2-6} 
 & MAT & TTT & (5,10) & 100 & 15962.98 \\* \midrule
\multirow{2}{*}{\begin{tabular}[c]{@{}l@{}}RSA\_BSAFE\_C\_4.0.4\_server\_regular.dot\\ INPUT noise at 0.05\%\end{tabular}} & \CEAL & RS & (5,10) & 100 & 11994.71 \\* \cmidrule(l){2-6} 
 & MAT & TTT & (10,20) & 100 & 34042.37 \\* \midrule
\multirow{2}{*}{\begin{tabular}[c]{@{}l@{}}RSA\_BSAFE\_C\_4.0.4\_server\_regular.dot\\ OUTPUT noise at 0.05\%\end{tabular}} & \CEAL & RS & (5,10) & 100 & 12175.24 \\* \cmidrule(l){2-6} 
 & MAT & TTT & (10,20) & 100 & 35512.27 \\* \midrule
\multirow{2}{*}{\begin{tabular}[c]{@{}l@{}}RSA\_BSAFE\_C\_4.0.4\_server\_regular.dot\\ INPUT noise at 0.1\%\end{tabular}} & \CEAL & RS & (5,10) & 100 & 20307.92 \\* \cmidrule(l){2-6} 
 & MAT & TTT & (20,30) & 100 & 59977.7 \\* \midrule
\multirow{2}{*}{\begin{tabular}[c]{@{}l@{}}RSA\_BSAFE\_C\_4.0.4\_server\_regular.dot\\ OUTPUT noise at 0.1\%\end{tabular}} & \CEAL & RS & (5,10) & 100 & 15649.11 \\* \cmidrule(l){2-6} 
 & MAT & TTT & (10,20) & 100 & 38083.16 \\* \midrule
\multirow{2}{*}{\begin{tabular}[c]{@{}l@{}}OpenSSL\_1.0.1g\_client\_regular.dot\\ INPUT noise at 0.01\%\end{tabular}} & \CEAL & RS & (5,10) & 100 & 11092.08 \\* \cmidrule(l){2-6} 
 & MAT & TTT & (5,10) & 100 & 17154.8 \\* \midrule
\multirow{2}{*}{\begin{tabular}[c]{@{}l@{}}OpenSSL\_1.0.1g\_client\_regular.dot\\ OUTPUT noise at 0.01\%\end{tabular}} & \CEAL & RS & (5,10) & 100 & 11104.54 \\* \cmidrule(l){2-6} 
 & MAT & TTT & (5,10) & 100 & 16591.06 \\* \midrule
\multirow{2}{*}{\begin{tabular}[c]{@{}l@{}}OpenSSL\_1.0.1g\_client\_regular.dot\\ INPUT noise at 0.05\%\end{tabular}} & \CEAL & RS & (5,10) & 100 & 12026.39 \\* \cmidrule(l){2-6} 
 & MAT & TTT & (10,20) & 100 & 38671.17 \\* \midrule
\multirow{2}{*}{\begin{tabular}[c]{@{}l@{}}OpenSSL\_1.0.1g\_client\_regular.dot\\ OUTPUT noise at 0.05\%\end{tabular}} & \CEAL & RS & (5,10) & 100 & 12157.83 \\* \cmidrule(l){2-6} 
 & MAT & TTT & (10,20) & 100 & 42172.35 \\* \midrule
\multirow{2}{*}{\begin{tabular}[c]{@{}l@{}}OpenSSL\_1.0.1g\_client\_regular.dot\\ INPUT noise at 0.1\%\end{tabular}} & \CEAL & RS & (5,10) & 100 & 22434.34 \\* \cmidrule(l){2-6} 
 & MAT & TTT & (20,30) & 100 & 81388.63 \\* \midrule
\multirow{2}{*}{\begin{tabular}[c]{@{}l@{}}OpenSSL\_1.0.1g\_client\_regular.dot\\ OUTPUT noise at 0.1\%\end{tabular}} & \CEAL & RS & (5,10) & 100 & 16134.44 \\* \cmidrule(l){2-6} 
 & MAT & TTT & (20,30) & 99 & 93221.85 \\* \midrule
\multirow{2}{*}{\begin{tabular}[c]{@{}l@{}}OpenSSL\_1.0.1l\_server\_regular.dot\\ INPUT noise at 0.01\%\end{tabular}} & MAT & KV & (5,10) & 100 & 83568.33 \\* \cmidrule(l){2-6} 
 & \CEAL & KV & (5,10) & 100 & 178480.85 \\* \midrule
\multirow{2}{*}{\begin{tabular}[c]{@{}l@{}}OpenSSL\_1.0.1l\_server\_regular.dot\\ OUTPUT noise at 0.01\%\end{tabular}} & MAT & KV & (5,10) & 100 & 76702.16 \\* \cmidrule(l){2-6} 
 & \CEAL & TTT & (5,10) & 100 & 170014.66 \\* \midrule
\multirow{2}{*}{\begin{tabular}[c]{@{}l@{}}OpenSSL\_1.0.1l\_server\_regular.dot\\ INPUT noise at 0.05\%\end{tabular}} & \CEAL & KV & (5,10) & 100 & 168956.14 \\* \cmidrule(l){2-6} 
 & MAT & TTT & (10,20) & 100 & 219047.93 \\* \midrule
\multirow{2}{*}{\begin{tabular}[c]{@{}l@{}}OpenSSL\_1.0.1l\_server\_regular.dot\\ OUTPUT noise at 0.05\%\end{tabular}} & MAT & KV & (10,20) & 100 & 166207.58 \\* \cmidrule(l){2-6} 
 & \CEAL & KV & (5,10) & 100 & 201153.28 \\* \midrule
\multirow{2}{*}{\begin{tabular}[c]{@{}l@{}}OpenSSL\_1.0.1l\_server\_regular.dot\\ INPUT noise at 0.1\%\end{tabular}} & MAT & KV & (20,30) & 100 & 400721.07 \\* \cmidrule(l){2-6} 
 & \CEAL & KV & (10,20) & 100 & 416222.86 \\* \midrule
\multirow{2}{*}{\begin{tabular}[c]{@{}l@{}}OpenSSL\_1.0.1l\_server\_regular.dot\\ OUTPUT noise at 0.1\%\end{tabular}} & \CEAL & KV & (10,20) & 100 & 662679.52 \\* \cmidrule(l){2-6} 
 & MAT & KV & (20,30) & 94 & 467222.21 \\* \midrule
\multirow{2}{*}{\begin{tabular}[c]{@{}l@{}}GnuTLS\_3.3.8\_client\_regular.dot\\ INPUT noise at 0.01\%\end{tabular}} & MAT & KV & (5,10) & 100 & 1074206.02 \\* \cmidrule(l){2-6} 
 & \CEAL & KV & (5,10) & 100 & 2270972.48 \\* \midrule
\multirow{2}{*}{\begin{tabular}[c]{@{}l@{}}GnuTLS\_3.3.8\_client\_regular.dot\\ OUTPUT noise at 0.01\%\end{tabular}} & MAT & KV & (5,10) & 100 & 1157429.8 \\* \cmidrule(l){2-6} 
 & \CEAL & KV & (5,10) & 100 & 2363794.93 \\* \midrule
\multirow{2}{*}{\begin{tabular}[c]{@{}l@{}}GnuTLS\_3.3.8\_client\_regular.dot\\ INPUT noise at 0.05\%\end{tabular}} & \CEAL & KV & (5,10) & 100 & 1547844.22 \\* \cmidrule(l){2-6} 
 & MAT & TTT & (10,20) & 100 & 3137292.43 \\* \midrule
\multirow{2}{*}{\begin{tabular}[c]{@{}l@{}}GnuTLS\_3.3.8\_client\_regular.dot\\ OUTPUT noise at 0.05\%\end{tabular}} & \CEAL & KV & (5,10) & 100 & 2251216.77 \\* \cmidrule(l){2-6} 
 & MAT & LSHARP & (20,30) & 100 & 5885067.18 \\* \midrule
\multirow{2}{*}{\begin{tabular}[c]{@{}l@{}}GnuTLS\_3.3.8\_client\_regular.dot\\ INPUT noise at 0.1\%\end{tabular}} & \CEAL & TTT & (10,20) & 100 & 6898721.62 \\* \cmidrule(l){2-6} 
 & MAT & KV & (20,30) & 88 & 4581787.99 \\* \midrule
\multirow{2}{*}{\begin{tabular}[c]{@{}l@{}}GnuTLS\_3.3.8\_client\_regular.dot\\ OUTPUT noise at 0.1\%\end{tabular}} & \CEAL & KV & (20,30) & 100 & 7484163.22 \\* \cmidrule(l){2-6} 
 & MAT & KV & (20,30) & 18 & 1744942.17 \\* \midrule
\multirow{2}{*}{\begin{tabular}[c]{@{}l@{}}NSS\_3.17.4\_client\_full.dot\\ INPUT noise at 0.01\%\end{tabular}} & MAT & KV & (5,10) & 100 & 151067.11 \\* \cmidrule(l){2-6} 
 & \CEAL & TTT & (5,10) & 100 & 254432.67 \\* \midrule
\multirow{2}{*}{\begin{tabular}[c]{@{}l@{}}NSS\_3.17.4\_client\_full.dot\\ OUTPUT noise at 0.01\%\end{tabular}} & MAT & LSHARP & (5,10) & 100 & 120249.1 \\* \cmidrule(l){2-6} 
 & \CEAL & LSHARP & (5,10) & 100 & 205706.28 \\* \midrule
\multirow{2}{*}{\begin{tabular}[c]{@{}l@{}}NSS\_3.17.4\_client\_full.dot\\ INPUT noise at 0.05\%\end{tabular}} & \CEAL & TTT & (5,10) & 100 & 212075.56 \\* \cmidrule(l){2-6} 
 & MAT & TTT & (10,20) & 100 & 220689.75 \\* \midrule
\multirow{2}{*}{\begin{tabular}[c]{@{}l@{}}NSS\_3.17.4\_client\_full.dot\\ OUTPUT noise at 0.05\%\end{tabular}} & \CEAL & TTT & (5,10) & 100 & 155770.93 \\* \cmidrule(l){2-6} 
 & MAT & TTT & (10,20) & 100 & 215655.37 \\* \midrule
\multirow{2}{*}{\begin{tabular}[c]{@{}l@{}}NSS\_3.17.4\_client\_full.dot\\ INPUT noise at 0.1\%\end{tabular}} & MAT & TTT & (20,30) & 100 & 502452.08 \\* \cmidrule(l){2-6} 
 & \CEAL & TTT & (10,20) & 100 & 532995.13 \\* \midrule
\multirow{2}{*}{\begin{tabular}[c]{@{}l@{}}NSS\_3.17.4\_client\_full.dot\\ OUTPUT noise at 0.1\%\end{tabular}} & \CEAL & KV & (5,10) & 100 & 277081.72 \\* \cmidrule(l){2-6} 
 & MAT & TTT & (20,30) & 97 & 542911.65 \\* \midrule
\multirow{2}{*}{\begin{tabular}[c]{@{}l@{}}OpenSSL\_1.0.1j\_server\_regular.dot\\ INPUT noise at 0.01\%\end{tabular}} & MAT & KV & (5,10) & 100 & 98860.19 \\* \cmidrule(l){2-6} 
 & \CEAL & TTT & (5,10) & 100 & 187712.24 \\* \midrule
\multirow{2}{*}{\begin{tabular}[c]{@{}l@{}}OpenSSL\_1.0.1j\_server\_regular.dot\\ OUTPUT noise at 0.01\%\end{tabular}} & MAT & KV & (5,10) & 100 & 114944.42 \\* \cmidrule(l){2-6} 
 & \CEAL & TTT & (5,10) & 100 & 192309.94 \\* \midrule
\multirow{2}{*}{\begin{tabular}[c]{@{}l@{}}OpenSSL\_1.0.1j\_server\_regular.dot\\ INPUT noise at 0.05\%\end{tabular}} & \CEAL & KV & (5,10) & 100 & 193923.29 \\* \cmidrule(l){2-6} 
 & MAT & KV & (10,20) & 100 & 195054.36 \\* \midrule
\multirow{2}{*}{\begin{tabular}[c]{@{}l@{}}OpenSSL\_1.0.1j\_server\_regular.dot\\ OUTPUT noise at 0.05\%\end{tabular}} & MAT & KV & (10,20) & 100 & 232477.71 \\* \cmidrule(l){2-6} 
 & \CEAL & KV & (5,10) & 100 & 234364.67 \\* \midrule
\multirow{2}{*}{\begin{tabular}[c]{@{}l@{}}OpenSSL\_1.0.1j\_server\_regular.dot\\ INPUT noise at 0.1\%\end{tabular}} & MAT & KV & (20,30) & 100 & 427809.6 \\* \cmidrule(l){2-6} 
 & \CEAL & LSHARP & (10,20) & 100 & 488676.44 \\* \midrule
\multirow{2}{*}{\begin{tabular}[c]{@{}l@{}}OpenSSL\_1.0.1j\_server\_regular.dot\\ OUTPUT noise at 0.1\%\end{tabular}} & \CEAL & KV & (10,20) & 100 & 779067.68 \\* \cmidrule(l){2-6} 
 & MAT & TTT & (20,30) & 83 & 501805.19 \\* \midrule
\multirow{2}{*}{\begin{tabular}[c]{@{}l@{}}GnuTLS\_3.3.8\_server\_regular.dot\\ INPUT noise at 0.01\%\end{tabular}} & MAT & RS & (5,10) & 100 & 305769.92 \\* \cmidrule(l){2-6} 
 & \CEAL & RS & (5,10) & 100 & 583269.31 \\* \midrule
\multirow{2}{*}{\begin{tabular}[c]{@{}l@{}}GnuTLS\_3.3.8\_server\_regular.dot\\ OUTPUT noise at 0.01\%\end{tabular}} & MAT & RS & (5,10) & 100 & 317981.5 \\* \cmidrule(l){2-6} 
 & \CEAL & RS & (5,10) & 100 & 565932.75 \\* \midrule
\multirow{2}{*}{\begin{tabular}[c]{@{}l@{}}GnuTLS\_3.3.8\_server\_regular.dot\\ INPUT noise at 0.05\%\end{tabular}} & MAT & RS & (10,20) & 100 & 593180.14 \\* \cmidrule(l){2-6} 
 & \CEAL & KV & (5,10) & 100 & 712136.47 \\* \midrule
\multirow{2}{*}{\begin{tabular}[c]{@{}l@{}}GnuTLS\_3.3.8\_server\_regular.dot\\ OUTPUT noise at 0.05\%\end{tabular}} & \CEAL & KV & (5,10) & 100 & 978550.49 \\* \cmidrule(l){2-6} 
 & MAT & LSHARP & (10,20) & 100 & 1124654.69 \\* \midrule
\multirow{2}{*}{\begin{tabular}[c]{@{}l@{}}GnuTLS\_3.3.8\_server\_regular.dot\\ INPUT noise at 0.1\%\end{tabular}} & \CEAL & KV & (10,20) & 100 & 1526791.66 \\* \cmidrule(l){2-6} 
 & MAT & RS & (20,30) & 98 & 1425113.07 \\* \midrule
\multirow{2}{*}{\begin{tabular}[c]{@{}l@{}}GnuTLS\_3.3.8\_server\_regular.dot\\ OUTPUT noise at 0.1\%\end{tabular}} & \CEAL & KV & (10,20) & 100 & 2954613.67 \\* \cmidrule(l){2-6} 
 & MAT & RS & (20,30) & 61 & 1067245.87 \\* \midrule
\multirow{2}{*}{\begin{tabular}[c]{@{}l@{}}TCP\_FreeBSD\_Client.dot\\ INPUT noise at 0.01\%\end{tabular}} & MAT & TTT & (5,10) & 100 & 51471.43 \\* \cmidrule(l){2-6} 
 & \CEAL & TTT & (5,10) & 100 & 83117.33 \\* \midrule
\multirow{2}{*}{\begin{tabular}[c]{@{}l@{}}TCP\_FreeBSD\_Client.dot\\ OUTPUT noise at 0.01\%\end{tabular}} & MAT & TTT & (5,10) & 100 & 49144.13 \\* \cmidrule(l){2-6} 
 & \CEAL & KV & (5,10) & 100 & 95920.54 \\* \midrule
\multirow{2}{*}{\begin{tabular}[c]{@{}l@{}}TCP\_FreeBSD\_Client.dot\\ INPUT noise at 0.05\%\end{tabular}} & \CEAL & KV & (10,20) & 100 & 263038.67 \\* \cmidrule(l){2-6} 
 & MAT & RS & (20,30) & 87 & 279107.21 \\* \midrule
\multirow{2}{*}{\begin{tabular}[c]{@{}l@{}}TCP\_FreeBSD\_Client.dot\\ OUTPUT noise at 0.05\%\end{tabular}} & \CEAL & KV & (10,20) & 100 & 344807.82 \\* \cmidrule(l){2-6} 
 & MAT & RS & (20,30) & 86 & 248240.19 \\* \midrule
\multirow{2}{*}{\begin{tabular}[c]{@{}l@{}}TCP\_FreeBSD\_Client.dot\\ INPUT noise at 0.1\%\end{tabular}} & \CEAL & TTT & (20,30) & 100 & 4443575.83 \\* \cmidrule(l){2-6} 
 & MAT & RS & (20,30) & 22 & 149063.14 \\* \midrule
\multirow{2}{*}{\begin{tabular}[c]{@{}l@{}}TCP\_FreeBSD\_Client.dot\\ OUTPUT noise at 0.1\%\end{tabular}} & \CEAL & TTT & (20,30) & 100 & 6302307.48 \\* \cmidrule(l){2-6} 
 & MAT & RS & (20,30) & 11 & 135983.09 \\* \midrule
\multirow{2}{*}{\begin{tabular}[c]{@{}l@{}}TCP\_Windows8\_Client.dot\\ INPUT noise at 0.01\%\end{tabular}} & MAT & TTT & (5,10) & 100 & 48855.27 \\* \cmidrule(l){2-6} 
 & \CEAL & LSHARP & (5,10) & 100 & 78782.79 \\* \midrule
\multirow{2}{*}{\begin{tabular}[c]{@{}l@{}}TCP\_Windows8\_Client.dot\\ OUTPUT noise at 0.01\%\end{tabular}} & MAT & TTT & (5,10) & 100 & 48247 \\* \cmidrule(l){2-6} 
 & \CEAL & KV & (5,10) & 100 & 90816.92 \\* \midrule
\multirow{2}{*}{\begin{tabular}[c]{@{}l@{}}TCP\_Windows8\_Client.dot\\ INPUT noise at 0.05\%\end{tabular}} & \CEAL & KV & (5,10) & 100 & 95279.01 \\* \cmidrule(l){2-6} 
 & MAT & RS & (10,20) & 100 & 157329.34 \\* \midrule
\multirow{2}{*}{\begin{tabular}[c]{@{}l@{}}TCP\_Windows8\_Client.dot\\ OUTPUT noise at 0.05\%\end{tabular}} & \CEAL & TTT & (5,10) & 100 & 94385.49 \\* \cmidrule(l){2-6} 
 & MAT & KV & (10,20) & 100 & 109640.61 \\* \midrule
\multirow{2}{*}{\begin{tabular}[c]{@{}l@{}}TCP\_Windows8\_Client.dot\\ INPUT noise at 0.1\%\end{tabular}} & \CEAL & TTT & (10,20) & 100 & 241756.52 \\* \cmidrule(l){2-6} 
 & MAT & TTT & (20,30) & 93 & 214732.66 \\* \midrule
\multirow{2}{*}{\begin{tabular}[c]{@{}l@{}}TCP\_Windows8\_Client.dot\\ OUTPUT noise at 0.1\%\end{tabular}} & \CEAL & TTT & (10,20) & 100 & 236025.16 \\* \cmidrule(l){2-6} 
 & MAT & RS & (20,30) & 93 & 319211.55 \\* \midrule
\multirow{2}{*}{\begin{tabular}[c]{@{}l@{}}GnuTLS\_3.3.8\_client\_full.dot\\ INPUT noise at 0.01\%\end{tabular}} & MAT & KV & (5,10) & 100 & 6374596.78 \\* \cmidrule(l){2-6} 
 & \CEAL & KV & (5,10) & 100 & 9383499.02 \\* \midrule
\multirow{2}{*}{\begin{tabular}[c]{@{}l@{}}GnuTLS\_3.3.8\_client\_full.dot\\ OUTPUT noise at 0.01\%\end{tabular}} & MAT & TTT & (5,10) & 100 & 7192497.09 \\* \cmidrule(l){2-6} 
 & \CEAL & KV & (5,10) & 100 & 10209832.74 \\* \midrule
\multirow{2}{*}{\begin{tabular}[c]{@{}l@{}}GnuTLS\_3.3.8\_client\_full.dot\\ INPUT noise at 0.05\%\end{tabular}} & \CEAL & TTT & (10,20) & 100 & 20145049 \\* \cmidrule(l){2-6} 
 & MAT & KV & (20,30) & 100 & 23124047.98 \\* \midrule
\multirow{2}{*}{\begin{tabular}[c]{@{}l@{}}GnuTLS\_3.3.8\_client\_full.dot\\ OUTPUT noise at 0.05\%\end{tabular}} & \CEAL & RS & (10,20) & 100 & 22228894.2 \\* \cmidrule(l){2-6} 
 & MAT & RS & (20,30) & 100 & 26922957.92 \\* \midrule
\multirow{2}{*}{\begin{tabular}[c]{@{}l@{}}GnuTLS\_3.3.8\_client\_full.dot\\ INPUT noise at 0.1\%\end{tabular}} & \CEAL & RS & (20,30) & 100 & 42055100.5 \\* \cmidrule(l){2-6} 
 & MAT & RS & (20,30) & 56 & 21071401.11 \\* \midrule
\multirow{2}{*}{\begin{tabular}[c]{@{}l@{}}GnuTLS\_3.3.8\_client\_full.dot\\ OUTPUT noise at 0.1\%\end{tabular}} & \CEAL & TTT & (5,10) & 99 & 19156049.74 \\* \cmidrule(l){2-6} 
 & MAT & LSHARP & (20,30) & 3 & 4953119 \\* \midrule
\multirow{2}{*}{\begin{tabular}[c]{@{}l@{}}TCP\_Linux\_Client.dot\\ INPUT noise at 0.01\%\end{tabular}} & MAT & TTT & (5,10) & 100 & 215066.05 \\* \cmidrule(l){2-6} 
 & \CEAL & LSHARP & (5,10) & 100 & 325839.84 \\* \midrule
\multirow{2}{*}{\begin{tabular}[c]{@{}l@{}}TCP\_Linux\_Client.dot\\ OUTPUT noise at 0.01\%\end{tabular}} & MAT & TTT & (5,10) & 100 & 183551.34 \\* \cmidrule(l){2-6} 
 & \CEAL & TTT & (5,10) & 100 & 337599.48 \\* \midrule
\multirow{2}{*}{\begin{tabular}[c]{@{}l@{}}TCP\_Linux\_Client.dot\\ INPUT noise at 0.05\%\end{tabular}} & \CEAL & TTT & (10,20) & 100 & 1689401.93 \\* \cmidrule(l){2-6} 
 & MAT & RS & (20,30) & 51 & 802177.67 \\* \midrule
\multirow{2}{*}{\begin{tabular}[c]{@{}l@{}}TCP\_Linux\_Client.dot\\ OUTPUT noise at 0.05\%\end{tabular}} & \CEAL & KV & (20,30) & 100 & 1872055.85 \\* \cmidrule(l){2-6} 
 & MAT & RS & (20,30) & 49 & 801882.9 \\* \midrule
\multirow{2}{*}{\begin{tabular}[c]{@{}l@{}}TCP\_Linux\_Client.dot\\ INPUT noise at 0.1\%\end{tabular}} & \CEAL & TTT & (20,30) & 99 & 15218384.08 \\* \cmidrule(l){2-6} 
 & MAT & TTT & (20,30) & 1 & 105619 \\* \midrule
\multirow{2}{*}{\begin{tabular}[c]{@{}l@{}}TCP\_Linux\_Client.dot\\ OUTPUT noise at 0.1\%\end{tabular}} & \CEAL & TTT & (10,20) & 95 & 24065009.26 \\* \cmidrule(l){2-6} 
 & MAT & TTT & (10,20) & 1 & 57169 \\* \midrule
\multirow{2}{*}{\begin{tabular}[c]{@{}l@{}}GnuTLS\_3.3.8\_server\_full.dot\\ INPUT noise at 0.01\%\end{tabular}} & MAT & RS & (5,10) & 100 & 1257996.28 \\* \cmidrule(l){2-6} 
 & \CEAL & RS & (5,10) & 100 & 2066983.82 \\* \midrule
\multirow{2}{*}{\begin{tabular}[c]{@{}l@{}}GnuTLS\_3.3.8\_server\_full.dot\\ OUTPUT noise at 0.01\%\end{tabular}} & MAT & KV & (5,10) & 100 & 1425157.15 \\* \cmidrule(l){2-6} 
 & \CEAL & RS & (5,10) & 100 & 1850643.2 \\* \midrule
\multirow{2}{*}{\begin{tabular}[c]{@{}l@{}}GnuTLS\_3.3.8\_server\_full.dot\\ INPUT noise at 0.05\%\end{tabular}} & \CEAL & LSHARP & (5,10) & 100 & 1992378.22 \\* \cmidrule(l){2-6} 
 & MAT & RS & (10,20) & 100 & 2695293.81 \\* \midrule
\multirow{2}{*}{\begin{tabular}[c]{@{}l@{}}GnuTLS\_3.3.8\_server\_full.dot\\ OUTPUT noise at 0.05\%\end{tabular}} & \CEAL & KV & (5,10) & 100 & 3939240.45 \\* \cmidrule(l){2-6} 
 & MAT & KV & (20,30) & 100 & 6560567.49 \\* \midrule
\multirow{2}{*}{\begin{tabular}[c]{@{}l@{}}GnuTLS\_3.3.8\_server\_full.dot\\ INPUT noise at 0.1\%\end{tabular}} & \CEAL & LSHARP & (10,20) & 100 & 6213300.73 \\* \cmidrule(l){2-6} 
 & MAT & RS & (20,30) & 97 & 5499339.33 \\* \midrule
\multirow{2}{*}{\begin{tabular}[c]{@{}l@{}}GnuTLS\_3.3.8\_server\_full.dot\\ OUTPUT noise at 0.1\%\end{tabular}} & \CEAL & KV & (10,20) & 100 & 10583663.38 \\* \cmidrule(l){2-6} 
 & MAT & RS & (20,30) & 22 & 4574410 \\* \midrule
\multirow{2}{*}{\begin{tabular}[c]{@{}l@{}}OpenSSL\_1.0.1g\_server\_regular.dot\\ INPUT noise at 0.01\%\end{tabular}} & MAT & KV & (5,10) & 100 & 151437.67 \\* \cmidrule(l){2-6} 
 & \CEAL & KV & (5,10) & 100 & 278231.68 \\* \midrule
\multirow{2}{*}{\begin{tabular}[c]{@{}l@{}}OpenSSL\_1.0.1g\_server\_regular.dot\\ OUTPUT noise at 0.01\%\end{tabular}} & MAT & KV & (5,10) & 100 & 124770.08 \\* \cmidrule(l){2-6} 
 & \CEAL & KV & (5,10) & 100 & 239356.33 \\* \midrule
\multirow{2}{*}{\begin{tabular}[c]{@{}l@{}}OpenSSL\_1.0.1g\_server\_regular.dot\\ INPUT noise at 0.05\%\end{tabular}} & \CEAL & LSHARP & (5,10) & 100 & 219344.8 \\* \cmidrule(l){2-6} 
 & MAT & KV & (10,20) & 100 & 313288.07 \\* \midrule
\multirow{2}{*}{\begin{tabular}[c]{@{}l@{}}OpenSSL\_1.0.1g\_server\_regular.dot\\ OUTPUT noise at 0.05\%\end{tabular}} & \CEAL & KV & (5,10) & 100 & 227970.48 \\* \cmidrule(l){2-6} 
 & MAT & KV & (10,20) & 100 & 291552.03 \\* \midrule
\multirow{2}{*}{\begin{tabular}[c]{@{}l@{}}OpenSSL\_1.0.1g\_server\_regular.dot\\ INPUT noise at 0.1\%\end{tabular}} & \CEAL & KV & (10,20) & 100 & 609691.3 \\* \cmidrule(l){2-6} 
 & MAT & KV & (20,30) & 100 & 749551.04 \\* \midrule
\multirow{2}{*}{\begin{tabular}[c]{@{}l@{}}OpenSSL\_1.0.1g\_server\_regular.dot\\ OUTPUT noise at 0.1\%\end{tabular}} & \CEAL & KV & (10,20) & 100 & 756468.39 \\* \cmidrule(l){2-6} 
 & MAT & KV & (20,30) & 77 & 560775.84 \\* \midrule
\multirow{2}{*}{\begin{tabular}[c]{@{}l@{}}DropBear.dot\\ INPUT noise at 0.01\%\end{tabular}} & MAT & RS & (5,10) & 100 & 137649.68 \\* \cmidrule(l){2-6} 
 & \CEAL & LSHARP & (5,10) & 100 & 191024.45 \\* \midrule
\multirow{2}{*}{\begin{tabular}[c]{@{}l@{}}DropBear.dot\\ OUTPUT noise at 0.01\%\end{tabular}} & MAT & LSHARP & (5,10) & 100 & 118485.35 \\* \cmidrule(l){2-6} 
 & \CEAL & LSHARP & (5,10) & 100 & 172474.12 \\* \midrule
\multirow{2}{*}{\begin{tabular}[c]{@{}l@{}}DropBear.dot\\ INPUT noise at 0.05\%\end{tabular}} & \CEAL & TTT & (5,10) & 100 & 657186.2 \\* \cmidrule(l){2-6} 
 & MAT & RS & (20,30) & 96 & 657705.43 \\* \midrule
\multirow{2}{*}{\begin{tabular}[c]{@{}l@{}}DropBear.dot\\ OUTPUT noise at 0.05\%\end{tabular}} & \CEAL & TTT & (10,20) & 100 & 746915.07 \\* \cmidrule(l){2-6} 
 & MAT & TTT & (20,30) & 91 & 749198.88 \\* \midrule
\multirow{2}{*}{\begin{tabular}[c]{@{}l@{}}DropBear.dot\\ INPUT noise at 0.1\%\end{tabular}} & \CEAL & TTT & (20,30) & 100 & 3770482.81 \\* \cmidrule(l){2-6} 
 & MAT & RS & (20,30) & 2 & 333079.5 \\* \midrule
\multirow{2}{*}{\begin{tabular}[c]{@{}l@{}}DropBear.dot\\ OUTPUT noise at 0.1\%\end{tabular}} & \CEAL & TTT & (20,30) & 100 & 3992875.08 \\* \cmidrule(l){2-6} 
 & MAT & KV & (20,30) & 2 & 247372 \\* \midrule
\multirow{2}{*}{\begin{tabular}[c]{@{}l@{}}OpenSSH.dot\\ INPUT noise at 0.01\%\end{tabular}} & \CEAL & KV & (5,10) & 100 & 4955202.9 \\* \cmidrule(l){2-6} 
 & MAT & KV & (10,20) & 100 & 5742879.77 \\* \midrule
\multirow{2}{*}{\begin{tabular}[c]{@{}l@{}}OpenSSH.dot\\ OUTPUT noise at 0.01\%\end{tabular}} & MAT & RS & (10,20) & 100 & 5354862.63 \\* \cmidrule(l){2-6} 
 & \CEAL & KV & (5,10) & 100 & 5865961.19 \\* \midrule
\multirow{2}{*}{\begin{tabular}[c]{@{}l@{}}OpenSSH.dot\\ INPUT noise at 0.05\%\end{tabular}} & \CEAL & KV & (20,30) & 100 & 19169545.12 \\* \cmidrule(l){2-6} 
 & MAT & RS & (20,30) & 17 & 3483515.06 \\* \midrule
\multirow{2}{*}{\begin{tabular}[c]{@{}l@{}}OpenSSH.dot\\ OUTPUT noise at 0.05\%\end{tabular}} & \CEAL & KV & (20,30) & 100 & 23365791.81 \\* \cmidrule(l){2-6} 
 & MAT & RS & (20,30) & 8 & 3307257.12 \\* \midrule
\multirow{2}{*}{\begin{tabular}[c]{@{}l@{}}model4.dot\\ INPUT noise at 0.01\%\end{tabular}} & MAT & TTT & (5,10) & 100 & 196191.43 \\* \cmidrule(l){2-6} 
 & \CEAL & TTT & (5,10) & 100 & 274268.45 \\* \midrule
\multirow{2}{*}{\begin{tabular}[c]{@{}l@{}}model4.dot\\ OUTPUT noise at 0.01\%\end{tabular}} & MAT & TTT & (5,10) & 100 & 202228.35 \\* \cmidrule(l){2-6} 
 & \CEAL & TTT & (5,10) & 100 & 252335.04 \\* \midrule
\multirow{2}{*}{\begin{tabular}[c]{@{}l@{}}model4.dot\\ INPUT noise at 0.05\%\end{tabular}} & \CEAL & TTT & (10,20) & 100 & 908233.13 \\* \cmidrule(l){2-6} 
 & MAT & KV & (20,30) & 100 & 1196013.58 \\* \midrule
\multirow{2}{*}{\begin{tabular}[c]{@{}l@{}}model4.dot\\ OUTPUT noise at 0.05\%\end{tabular}} & \CEAL & TTT & (10,20) & 100 & 596577.7 \\* \cmidrule(l){2-6} 
 & MAT & TTT & (20,30) & 100 & 880001.68 \\* \midrule
\multirow{2}{*}{\begin{tabular}[c]{@{}l@{}}model4.dot\\ INPUT noise at 0.1\%\end{tabular}} & \CEAL & TTT & (20,30) & 100 & 6163432.08 \\* \cmidrule(l){2-6} 
 & MAT & RS & (20,30) & 29 & 2072063.1 \\* \midrule
\multirow{2}{*}{\begin{tabular}[c]{@{}l@{}}model4.dot\\ OUTPUT noise at 0.1\%\end{tabular}} & \CEAL & TTT & (20,30) & 100 & 2565861.45 \\* \cmidrule(l){2-6} 
 & MAT & TTT & (20,30) & 47 & 932215.09 \\* \midrule
\multirow{2}{*}{\begin{tabular}[c]{@{}l@{}}model1.dot\\ INPUT noise at 0.01\%\end{tabular}} & MAT & TTT & (5,10) & 100 & 108632.49 \\* \cmidrule(l){2-6} 
 & \CEAL & TTT & (5,10) & 100 & 145286.78 \\* \midrule
\multirow{2}{*}{\begin{tabular}[c]{@{}l@{}}model1.dot\\ OUTPUT noise at 0.01\%\end{tabular}} & MAT & TTT & (5,10) & 100 & 110363.06 \\* \cmidrule(l){2-6} 
 & \CEAL & TTT & (5,10) & 100 & 157779.41 \\* \midrule
\multirow{2}{*}{\begin{tabular}[c]{@{}l@{}}model1.dot\\ INPUT noise at 0.05\%\end{tabular}} & \CEAL & TTT & (10,20) & 100 & 457462.53 \\* \cmidrule(l){2-6} 
 & MAT & TTT & (20,30) & 98 & 537245.61 \\* \midrule
\multirow{2}{*}{\begin{tabular}[c]{@{}l@{}}model1.dot\\ OUTPUT noise at 0.05\%\end{tabular}} & \CEAL & TTT & (10,20) & 100 & 852299.13 \\* \cmidrule(l){2-6} 
 & MAT & RS & (20,30) & 76 & 745257.8 \\* \midrule
\multirow{2}{*}{\begin{tabular}[c]{@{}l@{}}model1.dot\\ INPUT noise at 0.1\%\end{tabular}} & \CEAL & TTT & (20,30) & 100 & 3020350.5 \\* \cmidrule(l){2-6} 
 & MAT & RS & (20,30) & 15 & 723103.27 \\* \midrule
\multirow{2}{*}{\begin{tabular}[c]{@{}l@{}}model1.dot\\ OUTPUT noise at 0.1\%\end{tabular}} & \CEAL & TTT & (20,30) & 92 & 12212079.93 \\* \cmidrule(l){2-6} 
 & MAT & RS & (20,30) & 1 & 421758 \\* \midrule
\multirow{2}{*}{\begin{tabular}[c]{@{}l@{}}TCP\_Windows8\_Server.dot\\ INPUT noise at 0.01\%\end{tabular}} & MAT & RS & (10,20) & 100 & 1958923.06 \\* \cmidrule(l){2-6} 
 & \CEAL & TTT & (5,10) & 100 & 3732066.25 \\* \midrule
\multirow{2}{*}{\begin{tabular}[c]{@{}l@{}}TCP\_Windows8\_Server.dot\\ OUTPUT noise at 0.01\%\end{tabular}} & MAT & RS & (5,10) & 100 & 905053.04 \\* \cmidrule(l){2-6} 
 & \CEAL & TTT & (5,10) & 100 & 3473018.13 \\* \midrule
\multirow{2}{*}{\begin{tabular}[c]{@{}l@{}}TCP\_Windows8\_Server.dot\\ INPUT noise at 0.05\%\end{tabular}} & \CEAL & TTT & (10,20) & 100 & 11418676.36 \\* \cmidrule(l){2-6} 
 & MAT & RS & (20,30) & 96 & 4012156.69 \\* \midrule
\multirow{2}{*}{\begin{tabular}[c]{@{}l@{}}TCP\_Windows8\_Server.dot\\ OUTPUT noise at 0.05\%\end{tabular}} & MAT & RS & (20,30) & 100 & 4080149.32 \\* \cmidrule(l){2-6} 
 & \CEAL & KV & (10,20) & 100 & 7948493.19 \\* \midrule
\multirow{2}{*}{\begin{tabular}[c]{@{}l@{}}TCP\_Windows8\_Server.dot\\ OUTPUT noise at 0.1\%\end{tabular}} & \CEAL & TTT & (20,30) & 100 & 25669726.36 \\* \cmidrule(l){2-6} 
 & MAT & TTT & (20,30) & 7 & 3338306.29 \\* \midrule
\multirow{2}{*}{\begin{tabular}[c]{@{}l@{}}TCP\_FreeBSD\_Server.dot\\ INPUT noise at 0.01\%\end{tabular}} & \CEAL & TTT & (5,10) & 100 & 3139842.78 \\* \cmidrule(l){2-6} 
 & MAT & RS & (10,20) & 100 & 4470383.1 \\* \midrule
\multirow{2}{*}{\begin{tabular}[c]{@{}l@{}}TCP\_FreeBSD\_Server.dot\\ OUTPUT noise at 0.01\%\end{tabular}} & \CEAL & TTT & (5,10) & 100 & 3178107.17 \\* \cmidrule(l){2-6} 
 & MAT & TTT & (10,20) & 100 & 4242473.71 \\* \midrule
\begin{tabular}[c]{@{}l@{}}TCP\_FreeBSD\_Server.dot\\ INPUT noise at 0.05\%\end{tabular} & \CEAL & TTT & (20,30) & 100 & 19444312.79 \\* \midrule
\multirow{2}{*}{\begin{tabular}[c]{@{}l@{}}TCP\_FreeBSD\_Server.dot\\ OUTPUT noise at 0.05\%\end{tabular}} & \CEAL & TTT & (20,30) & 100 & 19933573.85 \\* \cmidrule(l){2-6} 
 & MAT & RS & (20,30) & 4 & 5323173.75 \\* \midrule
\multirow{2}{*}{\begin{tabular}[c]{@{}l@{}}TCP\_Linux\_Server.dot\\ INPUT noise at 0.01\%\end{tabular}} & MAT & RS & (10,20) & 100 & 3473612.19 \\* \cmidrule(l){2-6} 
 & \CEAL & TTT & (5,10) & 100 & 4344517.52 \\* \midrule
\multirow{2}{*}{\begin{tabular}[c]{@{}l@{}}TCP\_Linux\_Server.dot\\ OUTPUT noise at 0.01\%\end{tabular}} & MAT & RS & (10,20) & 100 & 3551066.4 \\* \cmidrule(l){2-6} 
 & \CEAL & TTT & (5,10) & 100 & 3687906.85 \\* \midrule
\multirow{2}{*}{\begin{tabular}[c]{@{}l@{}}TCP\_Linux\_Server.dot\\ INPUT noise at 0.05\%\end{tabular}} & \CEAL & TTT & (20,30) & 98 & 40306528.1 \\* \cmidrule(l){2-6} 
 & MAT & RS & (20,30) & 4 & 5661189 \\* \midrule
\multirow{2}{*}{\begin{tabular}[c]{@{}l@{}}TCP\_Linux\_Server.dot\\ OUTPUT noise at 0.05\%\end{tabular}} & \CEAL & TTT & (20,30) & 99 & 31894250.28 \\* \cmidrule(l){2-6} 
 & MAT & RS & (20,30) & 7 & 5675121.57 \\* \midrule
\multirow{2}{*}{\begin{tabular}[c]{@{}l@{}}model3.dot\\ INPUT noise at 0.01\%\end{tabular}} & MAT & TTT & (5,10) & 100 & 376951.11 \\* \cmidrule(l){2-6} 
 & \CEAL & TTT & (5,10) & 100 & 478557.28 \\* \midrule
\multirow{2}{*}{\begin{tabular}[c]{@{}l@{}}model3.dot\\ OUTPUT noise at 0.01\%\end{tabular}} & MAT & TTT & (5,10) & 100 & 352924.64 \\* \cmidrule(l){2-6} 
 & \CEAL & TTT & (5,10) & 100 & 460747.47 \\* \midrule
\multirow{2}{*}{\begin{tabular}[c]{@{}l@{}}model3.dot\\ INPUT noise at 0.05\%\end{tabular}} & \CEAL & TTT & (10,20) & 100 & 1502765.59 \\* \cmidrule(l){2-6} 
 & MAT & KV & (20,30) & 99 & 2196917.95 \\* \midrule
\multirow{2}{*}{\begin{tabular}[c]{@{}l@{}}model3.dot\\ OUTPUT noise at 0.05\%\end{tabular}} & MAT & TTT & (10,20) & 100 & 794580.92 \\* \cmidrule(l){2-6} 
 & \CEAL & TTT & (10,20) & 100 & 924573.34 \\* \midrule
\begin{tabular}[c]{@{}l@{}}model3.dot\\ INPUT noise at 0.1\%\end{tabular} & \CEAL & TTT & (20,30) & 31 & 37907809.55 \\* \midrule
\multirow{2}{*}{\begin{tabular}[c]{@{}l@{}}model3.dot\\ OUTPUT noise at 0.1\%\end{tabular}} & \CEAL & TTT & (10,20) & 100 & 2572752.19 \\* \cmidrule(l){2-6} 
 & MAT & TTT & (20,30) & 76 & 1730604 \\* \midrule
\multirow{2}{*}{\begin{tabular}[c]{@{}l@{}}BitVise.dot\\ INPUT noise at 0.01\%\end{tabular}} & \CEAL & TTT & (5,10) & 100 & 11049320.81 \\* \cmidrule(l){2-6} 
 & MAT & RS & (10,20) & 100 & 15036748.06 \\* \midrule
\multirow{2}{*}{\begin{tabular}[c]{@{}l@{}}BitVise.dot\\ OUTPUT noise at 0.01\%\end{tabular}} & \CEAL & KV & (5,10) & 100 & 11392077.23 \\* \cmidrule(l){2-6} 
 & MAT & TTT & (10,20) & 100 & 14827071.84 \\* \midrule
\multirow{2}{*}{\begin{tabular}[c]{@{}l@{}}BitVise.dot\\ INPUT noise at 0.05\%\end{tabular}} & \CEAL & KV & (20,30) & 91 & 69346667.59 \\* \cmidrule(l){2-6} 
 & MAT & RS & (20,30) & 1 & 6870614 \\* \midrule
\begin{tabular}[c]{@{}l@{}}BitVise.dot\\ OUTPUT noise at 0.05\%\end{tabular} & \CEAL & KV & (20,30) & 77 & 77534701.29 \\* \bottomrule
\end{longtable}
\section{Experimental Analysis: Other Findings}
\label{app:analysis}
We regroup here some considerations on our experimental results that are not necessary to answer the research question. 

\subsection{Revision ratio}
\CEAL requires that all the information is eventually re-checked to ensure that no (potentially noisy) observation can remain unchecked. In our case this is ensured by the m-complete test suites we used, because we consider a black-box checking with Mealy machines, where all prefixes of tested words are also tested (by causality). 

For other use cases, it can be necessary to ensure that in the \(\Test_{\BS}\) function, a given ratio $\gamma$ of tests is spent rechecking the oldest information stored in \(\BS\).

We implemented this ratio and duplicated (part of) the experiments with a revision ratio (proportion of tests spent rechecking old information) taken in $(0.2, 0.4, 0.6, 0.8, 0.99)$.
We found that in our specific experiments, \CEAL's revision ratio parameter ($\gamma$) did not provide statistically significant differences in success rates. In particular, for the cases of learning Mealy machines using nondeterministic equivalence testing methods such as Hybrid-ADS, we found that this parameter can be constantly set to 0. 
This may, however, not be the case were we to use different types of equivalence algorithms, or if we were not learning Mealy machines.

\subsection{Influence of the alphabet}
Target size is not a sole determining factor in learning difficulty. For example, in the success graphs for level 0.05\% and 0.1\% (Figs. \ref{fig:result-INPUT-0.05-success}, \ref{fig:result-INPUT-0.1-success}, \ref{fig:result-OUTPUT-0.05-success}, \ref{fig:result-OUTPUT-0.1-success}), we can see that the smallest 4 or so targets prove to be harder to learn with noise than the following dozen or so, despite being smaller. This is because these smaller targets have quite big alphabets, in fact their alphabets are twice the size of the following, well performing targets. This is expected as the alphabet size also increases the number of queries that need to be done, and thus the chance for noisy queries.

\subsection{Efficiency gain}
Looking at efficiency, we have also found that in the cases where system test counts can be compared (i.e., the resulting runs have meaningful success rates), \CEAL provides an average reduction of 6\% in test use, when compared to MAT. This suggests that \CEAL not only pushes the boundary of learning in noisy environments, it does so efficiently, even in cases where MAT is at its best.
}
\end{document}